\documentclass{article}
\usepackage{ijcai19}

\usepackage{times}
\usepackage{soul}
\usepackage{url}
\usepackage[hidelinks]{hyperref}
\usepackage[utf8]{inputenc}
\usepackage[small]{caption}
\usepackage{booktabs}
\usepackage{amsthm}
\usepackage{amsmath}
\usepackage{amssymb}

\usepackage{algorithm}
\usepackage{algorithmic}

\newtheorem{theorem}{Theorem}
\newtheorem{lemma}[theorem]{Lemma}

\newtheorem{corollary}[theorem]{Corollary}

\newtheorem{assump}{Assumption}
\newtheorem{remark}{Remark}
\newcommand\numberthis{\addtocounter{equation}{1}\tag{\theequation}}

\def\A{{\bf A}}

\def\a{{\bf a}}
\def\B{{\bf B}}

\def\C{{\bf C}}

\def\c{{\bf c}}

\def\g{{\bf g}}

\def\H{{\bf H}}
\def\I{{\bf I}}

\def\PP{{\bf P}}
\def\Q{{\bf Q}}

\def\p{{\bf p}}

\def\S{{\bf S}}
\def\s{{\bf s}}

\def\U{{\bf U}}
\def\u{{\bf u}}
\def\V{{\bf V}}
\def\v{{\bf v}}

\def\w{{\bf w}}

\def\x{{\bf x}}

\def\Z{{\bf Z}}
\def\z{{\bf z}}
\def\0{{\bf 0}}
\def\1{{\bf 1}}

\def\prox{\text{prox}}
\def\ed{{d_{\text{eff}}^{\gamma}}}

\def\OM{{\mathcal O}}

\def\SM{{\mathcal S}}

\def\RB{{\mathbb R}}

\def\PB{{\mathbb P}}

\def\TH{\widetilde{\bf H}}

\def\TA{\widetilde{\bf A}}

\def\Tp{\widetilde{\bf p}}

\def\Si{\mbox{\boldmath$\Sigma$\unboldmath}}

\def\Gam{\mbox{\boldmath$\Gamma$\unboldmath}}

\def\De{\mbox{\boldmath$\Delta$\unboldmath}}
\def\Ome{\mbox{\boldmath$\Omega$\unboldmath}}

\def\argmin{\mathop{\rm argmin}}

\def\nnz{\mathsf{nnz}}

\title{Do Subsampled Newton Methods Work for High-Dimensional Data?}

\author{
Xiang Li$^1$ \and Shusen Wang$^2$ \and Zhihua Zhang$^1$
\affiliations
$^1$School of Mathematical Sciences,
Peking University, China
\\
$^2$Department of Computer Science,
Stevens Institute of Technology, USA \\ 
{smslixiang@pku.edu.cn}, \quad
{shusen.wang@stevens.edu}, \quad
{zhzhang@math.pku.edu.cn}
}

\begin{document}
\maketitle

\begin{abstract}
Subsampled Newton methods approximate Hessian matrices through subsampling techniques, alleviating the cost of forming Hessian matrices but using sufficient curvature information. However, previous results require $\Omega (d)$ samples to approximate Hessians, where $d$ is the dimension of data points, making it less practically feasible for high-dimensional data. The situation is deteriorated when $d$ is comparably as large as the number of data points $n$, which requires to take the whole dataset into account, making subsampling useless. This paper theoretically justifies the effectiveness of subsampled Newton methods on convex empirical risk minimization with high dimensional data.  Specifically, we provably need only $\widetilde{\Theta}(\ed)$ samples the approximation of Hessian matrices, where $\ed$ is the $\gamma$-ridge leverage and can be much smaller than $d$ as long as $n\gamma \gg 1$. Additionally, we extend this result so that subsampled Newton methods can work for high-dimensional data on both distributed optimization problems and non-smooth regularized problems.
\end{abstract}

\section{Introduction}


Let $\x_1, ..., \x_n \in \RB^d$ be the feature vectors, 
$l_i(\cdot)$ is a convex, smooth, and twice differentiable loss function; the response $y_i$ is captured by $l_i$.
In this paper, we study the following optimization problem:
\begin{small}
\begin{equation}
\label{problem}
\min_{\w \in \RB^d} G(\w) 
\: := \: 
\frac{1}{n} \sum_{j=1}^n l_j (\x_j^T \w ) + \frac{\gamma }{2} \| \w\|_2^2 + r (\w )
\end{equation}
\end{small}%
where $r(\cdot)$ is a non-smooth convex function. We first consider the simple case where $r$ is zero, i.e.,
\begin{small}
\begin{equation}
\label{problem:simple}
\min_{\w \in \RB^d} F(\w) 
\: := \: 
\frac{1}{n} \sum_{j=1}^n l_j (\x_j^T \w ) + \frac{\gamma }{2} \| \w\|_2^2.
\end{equation}
\end{small}%
Such a convex optimization problem~\eqref{problem:simple} arises frequently in machining learning~\cite{shalev2014understanding}. For example, in logistic regression, $l_j(\x_j^{T}\w)=\log(1+\exp(-y_j \x_j^{T}\w))$, 
and in linear regression, $l_j(\x_j^{T}\w)= \frac{1}{2} ( \x_j^{T}\w - y_j )^2$. Then we consider the more general case where $r$ is non-zero, e.g., LASSO~\cite{tibshirani1996regression} and elastic net~\cite{zou2005regularization}.

To solve~\eqref{problem:simple}, many first order methods have been proposed.
First-order methods solely exploit information in the objective function and its gradient.
Accelerated gradient descent~\cite{golub2012matrix,nesterov2013introductory,bubeck2014theory}, 
stochastic gradient descent \cite{robbins1985stochastic}, and their variants \cite{lin2015universal,johnson2013accelerating,schmidt2017minimizing} 
are the most popular approaches in practice due to their simplicity and low per-iteration time complexity. As pointed out by \cite{xu2017second}, the downsides of first-order methods are the slow convergence to high-precision and the sensitivity to condition number and hyper-parameters.

Second-order methods use not only the gradient but also information in the Hessian matrix in their update. In particular, the Newton's method, a canonical second-order method, has the following update rule:
\begin{small}
\begin{equation}
\label{eq:newtonup}
\w_{t+1} = \w_{t} - \alpha_t \H_t^{-1}\g_t , 
\end{equation}
\end{small}%
where the gradient $\g_t = \nabla F (\w_t )$ is the first derivative of the objective function at $\w_t$,
the Hessian $\H_t =\nabla^2 F (\w_t ) $ is the second derivative at $\w_t$,
and $\alpha_t$ is the step size and can be safely set as one under certain conditions.
Compared to the first-order methods, Newton’s method requires fewer iterations and are more robust to the hyper-parameter setting, and guaranteed super-linear local convergence to high-precision.
However, Newton's method is slow in practice, as in each iteration many Hessian-vector products are required to solve the inverse problem $\H_t \p = \g_t$. Quasi-Newton methods use information from the history of updates to construct Hessian \cite{dennis1977quasi}. Well-known works include Broyden-Fletcher-Goldfarb-Shanno (BFGS) \cite{wright1999numerical} and its limited memory version (L-BFGS) \cite{liu1989limited}, but their convergence rates are not comparable to Newton's method.

Recent works proposed the Sub-Sampled Newton (SSN) methods to reduce the per-iteration complexity of the Newton's method \cite{byrd2011use,pilanci2015iterative,roosta2016sub,pilanci2017newton,xu2017second,berahas2017investigation,ye2017approximate}.
For the particular problem~\eqref{problem:simple}, the Hessian matrix can be written in the form
\begin{small}
	\begin{equation}
	\label{eq:pH}
	\H_t \: = \: \frac{1}{n} \A_t^T \A_t + \gamma \I_d ,
	\end{equation}
\end{small}%
for some $n\times d$ matrix $\A_t$ whose $i$-th row is a scaling of $\x_i$.
The basic idea of SSN is to sample and scale $s$ ($s \ll n$) rows of $\A$ to form $\TA_t \in \RB^{s\times d}$ and approximate $\H_t$ by 
\begin{small}
	\begin{equation*}
	\widetilde{\H}_t \: = \: \frac{1}{s} \TA_t^T \TA_t + \gamma \I_d ,
	\end{equation*}
\end{small}%
The quality of Hessian approximation is guaranteed by random matrix theories \cite{tropp2015introduction,woodruff2014sketching},
based on which the convergence rate of SSN is established.

As the second-order methods perform heavy computation in each iteration and converge in a small number of iterations, they have been adapted to solve distributed machine learning aiming at reducing the communication cost \cite{shamir2014communication,mahajan2015efficient,zhang2015disco,reddi2016aide,wang2018giant}.
In particular, the Globally Improved Approximate NewTon Method (GIANT) method is based on the same idea as SSN and has fast convergence rate.

As well as Newton's method, SSN is not directly applicable for~\eqref{problem} because the objective function is non-smooth. Following the proximal-Newton method~\cite{lee2014proximal}, 
SSN has been adapted to solve convex optimization with non-smooth regularization~\cite{liu2017inexact}. SSN has also been applied to optimize nonconvex problem~\cite{xu2017second,tripuraneni2017stochastic}.

\subsection{Our contributions}

Recall that $n$ is the total number of samples, $d$ is the number of features, and $s$ is the size of the randomly sampled subset. (Obviously $s \ll n$, otherwise the subsampling does not speed up computation.)
The existing theories of SSN require $s$ to be at least $\Omega (d )$.
For the big-data setting, i.e., $d \ll n$, the existing theories nicely guarantee the convergence of SSN.

However, high-dimensional data is not uncommon at all in machine learning;
$d$ can be comparable to or even greater than $n$.
Thus requiring both $s \ll n$ and $s=\Omega (d )$ seriously limits the application of SSN.
We considers the question:
\begin{center}
\textit{Do SSN and its variants work for~\eqref{problem} when $s < d$?}
\end{center}
The empirical studies in \cite{xu2016sub,xu2017second,wang2018giant} indicate that yes, SSN and its extensions have fast convergence even if $s$ is substantially smaller than $d$. However, their empirical observations are not supported by theory.

This work bridges the gap between theory and practice for convex empirical risk minimization. We show it suffices to use $s = \tilde{\Theta} ( \ed )$ uniformly sampled subset to approximate the Hessian, where $\gamma$ is the regularization parameter, $\ed$ ($\leq d$) is the $\gamma$-effective-dimension of the $d\times d$ Hessian matrix, and $\tilde{\Theta}$ hides the constant and logarithmic factors.
If $n\gamma$ is larger than most of the $d$ eigenvalues of the Hessian, then $\ed$ is tremendously smaller than $d$ \cite{cohen2015ridge}.
Our theory is applicable to three SSN methods.
\begin{itemize}
	\item 
	In Section~\ref{sec:ssn}, we study the convex and smooth problem~\eqref{problem:simple}.
	we show the convergence of the standard SSN with the effective-dimension dependence and improves \cite{xu2016sub}.
	\item
	In Section~\ref{sec:dist}, for the same optimization problem~\eqref{problem:simple}, we extend the result to the distributed computing setting and improves the bound of GIANT \cite{wang2018giant}.
	\item
	In Section~\ref{sec:sspn}, we study a convex but nonsmooth problem~\eqref{problem} and analyze the combination of SSN and proximal-Newton.
\end{itemize}
In Section~\ref{sec:inexact}, we analyse SSN methods when the subproblems are inexactly solved. The proofs of the main theorems are in the appendix.

\section{Notation and Preliminary}\label{sec:notation}

\paragraph{Basic matrix notation.}
Let $\I_n$ be the $n \times n$ indentity matrix.
Let $\| \a \|_2 $ denote the vector $\ell_2$ norm and $\|\A \|_2$ denote the matrix spectral norm.
Let 
\begin{small}
	\begin{equation} \label{eq:svd}
	\A =\U \Si \V^{T}=\sum_{i=1}^d \sigma_i\u_i\v_i^{T}
	\end{equation}
\end{small}%
be its singular value decomposition (SVD), with $\sigma_{\max}(\A ) $ its largest singular value and $\sigma_{\min}(\A ) $ the smallest (the $d$-th largest). 
The moore-Penrose inverse of $\A$ is defined by $\A^{\dagger} = \V \Si^{-1}\U^{T}$.
If a symmetric real matrix has no negative eigenvalues, it is called symmetric positive semidefinite (SPSD).
We denote $\A \preceq \B$ if $\B - \A$ is SPSD. For the SPD matrice $\H$, we define a norm by $\|\x\|_{\H} = \sqrt{\x^{T}\H \x}$ and its conditional number by $\kappa(\H)= \frac{\sigma_{\max}(\H)}{\sigma_{\min}(\H)}$.

\paragraph{Ridge leverage scores.}
For $\A=[\a_1^{T};\cdots;\a_n^{T}] \in \RB^{n \times d}$, its row $\gamma$-ridge leverage score ($\gamma \ge 0$) is defined by
\begin{small}
	\begin{equation}
	\label{eq:rls}
	l_j^{\gamma} = \a_j^{T}(\A^{T}\A + n\gamma \I_d)^{\dagger}\a_j 
	=\sum_{k=1}^d \frac{\sigma_k^2}{\sigma_k^2 + n\gamma}u_{jk}^2,
	\end{equation}
\end{small}%
for $j \in [n] \triangleq \{1,2,...,n\}$.
Here $\sigma_k$ and $\u_k$ are defined in \eqref{eq:svd}.
For $\gamma=0$, $ l_j^{\gamma}$ is the standard leverage score used by~\cite{drineas2008cur,mahoney2011ramdomized}.

\paragraph{Effective dimension.}
The $\gamma$-effective dimension of $\A \in \RB^{n \times d}$  is defined by
\begin{small}
	\begin{equation}
	\label{eq:gamma-rls}
	\ed(\A ) 
	\: = \: \sum_{j=1}^n  l_j^{\gamma}  
	\: = \: \sum_{k=1}^d \frac{\sigma_k^2}{\sigma_k^2 + n\gamma}
	\: \leq \: d .
	\end{equation}
\end{small}%
If $n\gamma$ is larger than most of the singular values of $\A^T \A$, then $\ed(\A ) $ is tremendously smaller than $d$~\cite{alaoui2015fast,cohen2017input}.
In fact, to trade-off the bias and variance, the optimal setting of $\gamma$ makes $n\gamma$ comparable to the top singular values of $\A^T \A$~\cite{hsu2014random,wang2017sketched}, and thus $\ed(\A) $  is small in practice.

\paragraph{Ridge coherence.}
The row $\gamma$-ridge coherence of $\A \in \RB^{n \times d}$ is 
\begin{small}
\begin{equation}
\label{eq:row coherence}
\mu^{\gamma} \: = \: \frac{n}{\ed} \, \max_{i \in [n]} l_i^{\gamma},
\end{equation}
\end{small}%
which measures the extent to which the information in the rows concentrates. 
If $\A$ has most of its mass in a relatively small number of rows, its $\gamma$-ridge coherence could be high. 
This concept is necessary for matrix approximation via uniform sampling.
It could be imagined that if most information is in a few rows, which means high coherence, then uniform sampling is likely to miss some of the important rows, leading to low approximation quality.
When $\gamma=0$, it coincides with the standard row coherence 
\begin{small}
	\begin{equation*}
	\mu^0 
	\: = \: \frac{n}{d} \, \max_{j \in [n]}l_j^0
	\: = \: \frac{n}{d} \, \max_{j \in [n]}\a_j^{T}(\A^{T}\A)^{\dagger}\a_j
	\end{equation*}
\end{small}%
which is widely used to analyze techniques such as compressed
sensing~\cite{candes2006stable}, matrix completion~\cite{candes2009exact}, robust PCA~\cite{candes2011robust} and so on.

\paragraph{Gradient and Hessian.}
For the optimization problem \eqref{problem:simple}, the gradient of $F (\cdot )$ at $\w_t$ is 
\begin{small}
	\begin{equation*}
	\g_t 
	\: = \: \frac{1}{n} \sum_{j=1}^{n} l_j' (\x_j^T \w_t) \cdot \x_j \, + \, \gamma \w_t 
	\: \in \: \RB^d .
	\end{equation*}
\end{small}%
The Hessian matrix at $\w_t$ is 
\begin{small}
	\begin{equation*}
	\H_t 
	\: = \: \frac{1}{n} \sum_{j=1}^{n} l_j'' (\x_j^T \w_t) \cdot \x_j \x_j^T \, + \, \gamma \I_d
	\: \in \: \RB^{d\times d} .
	\end{equation*}
\end{small}%
Let $\a_j = \sqrt{l_j^{''}(\x_i^{T}\w_t)} \cdot \x_j \in \RB^d$ and 
\begin{small}
	\begin{equation}
	\label{eq:A}
	\A_t 
	\: = \: [\a_1, \cdots , \a_n]^{T}
	\: \in \: \RB^{n \times d}.
	\end{equation}
\end{small}%
In this way, the Hessian matrix can be expressed as 
\begin{small}
	\begin{equation}
	\label{eq:truehessian}
	\H_t 
	\: = \: \tfrac{1}{n} \A_t^T \A_t + \gamma \I_d
	\: \in \: \RB^{d\times d} .
	\end{equation}
\end{small}%

\section{Sub-Sampled Newton (SSN)}\label{sec:ssn}

In this section, we provide new and stronger convergence guarantees for the SSN methods.
For SSN with uniform sampling, we require a subsample size of $s = \tilde{\Theta} (\mu^\gamma \ed)$;
For SSN with ridge leverage score sampling,\footnote{We do not describe the ridge leverage score sampling in detail; the readers can refer to~\cite{alaoui2015fast,cohen2015ridge}.}
a smaller sample size, $s = \tilde{\Theta} (\ed)$, suffices.
Because $\ed$ is typically much smaller than $d$, our new results guarantee convergence when $s < d$.

\subsection{Algorithm description}

We set an interger $s$ ($\ll n$) and uniformly sample $s$ items out of $[n]$ to form the subset $\SM$.
In the $t$-th iteration, we form the matrix $\tilde{\A}_t \in \RB^{s\times d}$ which contains the rows of $\A_t \in \RB^{n\times d}$ indexed by $\SM$ and the full gradient $\g_t$.
Then, the approximately Newton direction $\tilde{\p}_t$ is computed by solving the linear system
\begin{small}
	\begin{equation} \label{eq:ssn_linear_system}
	\big( \tfrac{1}{s} \TA_t \TA_t^T + \gamma \I_d \big) \, \p \: = \: \g_t
	\end{equation}
\end{small}%
by either matrix inversion or the conjugate gradient.
Finally, $\w$ is updated by 
\begin{small}
	\begin{equation*}
	\w_{t+1} \: = \: \w_t - \alpha_t \tilde{\p}_t ,
	\end{equation*}
\end{small}%
where $\alpha_t$ can be set to one or found by line search.
In the rest of this section, we only consider $\alpha_t = 1$.

Most of the computation is performed in solving \eqref{eq:ssn_linear_system}.
The only difference between the standard Newton and the SSN methods is replacing $\A_t \in \RB^{n\times d}$ by $\tilde{\A}_t \in \RB^{s\times d}$.
Compared to Newton's method, SSN leads to an almost $\frac{n}{s}$-factor speed up of the per-iteration computation;
however, SSN requires more iterations to converge.
Nevertheless, to reach a fixed precision, the overall cost of SSN is much lower than Newton's method.

\subsection{Our improved convergence bounds}

\paragraph{Improved bound for quadratic loss.}
We let $\w^\star$ be the unique (due to the strong convexity) optimal solution to the problem~\ref{problem},
$\w_t$ be the intermediate output of the $t$-th iteration,
and $\De_t = \w_t - \w^\star$.
If the loss function of \eqref{problem} is quadratic, e.g., $l_j (\x_j^T \w) = \frac{1}{2} (\x_j^T \w - y_j)^2$, the Hessian matrix $\H_t = \frac{1}{n} \A_t^T \A_t + \gamma \I_d$ does not change with the iteration, so we use $\H$ and $\A$ instead.
Theorem~\ref{thm:ssn_quad} guarantees the global convergence of SSN.

\begin{theorem}[Global Convergence]
	\label{thm:ssn_quad}
	Let $\ed$ and $ \mu^{\gamma}$ respectively be the $\gamma$-ridge leverage score and $\gamma$-coherence of $\A$, and $\kappa$ be the condition number of $\H$.
	Let $\varepsilon \in (0, \frac{1}{4})$ and $\delta \in (0, 1)$ be any user-specified constants.
	Assume the loss function of \eqref{problem} is quadratic.
	For a sufficiently large sub-sample size:
	\begin{small}
		\begin{equation*}
		s \: = \: \Theta \Big(  \tfrac{\mu^{\gamma}\ed}{\varepsilon^2} \log \tfrac{\ed}{\delta} \Big)  ,
		\end{equation*}
	\end{small}%
	with probability at least $ 1-\delta$,
	\begin{small}
		\begin{equation} \label{eq:ssn_global}
		\big\| \De_t \big\|_2
		\: \leq \: \epsilon^t \sqrt{\kappa}  \, \big\| \De_0 \big\|_2 .
		\end{equation}
	\end{small}
\end{theorem}

\begin{proof}
	We prove the theorem in Appendix~\ref{ap:snn_global}.
\end{proof}

\paragraph{Improved bound for non-quadratic loss.}
If the loss function of \eqref{problem} is non-quadratic, the Hessian matrix $\H_t$ changes with iteration, and we can only guarantee fast local convergence, as well as the prior works~\cite{roosta2016sub,xu2016sub}.
We make a standard assumption on the Hessian matrix, which is required by all the prior works on Newton-type methods.

\begin{assump}
	\label{am:lipschitz}
	The Hessian matrix $\nabla^2 F(\w)$ is $L$-Lipschitz continuous, i.e., $\| \nabla^2 F(\w) - \nabla^2 F(\w' ) \|_2 \le L \| \w - \w'\|_2$,
	for arbitrary $\w$ and $\w'$.
\end{assump}

\begin{theorem} [Local Convergence]
	\label{thm:ssn_nonquad}
	Let $\ed, \mu^{\gamma}$ respectively be the $\gamma$-ridge leverage score and $\gamma$-coherence of $\A_t$.
	Let $\varepsilon \in (0, \frac{1}{4})$ and $\delta \in (0, 1)$ be any user-specified constants.
	Let Assumption~\ref{am:lipschitz} be satisfied.
	For a sufficiently large sub-sample size:
	\begin{small}
		\begin{equation*}
		s \: = \: \Theta \Big(  \tfrac{\mu^{\gamma}\ed}{\varepsilon^2} \log \tfrac{\ed}{\delta} \Big)  ,
		\end{equation*}
	\end{small}%
	with probability at least $ 1-\delta$,
	\begin{small}
		\begin{equation} \label{eq:ssn_local}
		\big\|\De_{t+1} \big\|_2
		\: \leq \: \varepsilon \, \sqrt{ \kappa_t } \, \big\|\De_{t} \big\|_2
		+  \tfrac{L}{\sigma_{\min}(\H_t)} \, \big\|\De_{t} \big\|_2^2 ,
		\end{equation}
	\end{small}%
	where $\kappa_t = \tfrac{ \sigma_{\max} (\H_t) }{ \sigma_{\min} (\H_t) }$ is the condition number.
\end{theorem}

\begin{proof}
	We prove the theorem in Appendix~\ref{ap:snn_local}.
\end{proof}

\begin{theorem} \label{thm:ridge}
	If ridge leverage score sampling is used instead, the sample complexity in Theorems~\ref{thm:ssn_quad} and \ref{thm:ssn_nonquad} will be improved to 
	\begin{small}
		\begin{equation*}
		s \: = \: \Theta \Big(  \tfrac{\ed}{\varepsilon^2} \log \tfrac{\ed}{\delta} \Big) .
		\end{equation*}
	\end{small}%
\end{theorem}

\begin{remark}
	Ridge leverage score sampling eliminates the dependence on the coherence, and the bound is stronger than all the existing sample complexities for SSN.
	We prove the corollary in Appendix~\ref{ap:ssn:ridge}.
	However, the ridge leverage score sampling is expensive and impractical and thus has only theoretical interest.
\end{remark}

Although Newton-type methods empirically demonstrate fast global convergence in almost all the real-world applications, they do not have strong global convergence guarantee.
A weak global convergence bound for SSN was established by~\cite{roosta2016sub}.
We do not further discuss the global convergence issue in this paper.

\subsection{Comparison with prior work}

For SSN with uniform sampling, the prior work~\cite{roosta2016sub} showed that to obtain the same convergence bounds as ours, \eqref{eq:ssn_global} and \eqref{eq:ssn_local},  the sample complexity should be
\begin{small}
	\begin{equation*}
	s \: = \: \Theta \Big(  \tfrac{n {\kappa_t} }{\varepsilon^2 (1- \epsilon \kappa_t)^2} 
	\, \tfrac{\max_{i} \| \a_i\|_2^2}{\|\A\|_2^2} 
	\, \log \tfrac{d}{\delta} \Big)  .
	\end{equation*}
\end{small}%
In comparison, to obtain a same convergence rate, our sample complexity has a better dependence on the condition number and the dimensionality.

For the row norm square sampling of~\cite{xu2016sub}, which is slightly more expensive than uniform sampling, a sample complexity of 
\begin{small}
	\begin{equation*}
	s \: = \: \tilde{\Theta} \Big( \tfrac{1 }{ \varepsilon^2 (1- \epsilon \kappa_t)^2} 
	\, \tfrac{ \sigma_{\max} (\A_t^T \A_t ) + n\gamma}{ \sigma_{\max} (\A_t^T \A_t )  }
	\, \sum_{i=1}^d \tfrac{ \sigma_{i} (\A_t^T \A_t )  }{ \sigma_{\min} (\A_t^T \A_t )+ n \gamma }  \Big) 
	\end{equation*}
\end{small}%
suffices for the same convergence rates as ours, \eqref{eq:ssn_global} and \eqref{eq:ssn_local}.
Their bound may or may not guarantee convergence for $s < d$.
Even if $n\gamma$ is larger than most of the singular values of $\A_t^T \A_t$, their required sample complexity can be large.

For leverage score sampling,~\cite{xu2016sub} showed that to obtain the same convergence bounds as ours, \eqref{eq:ssn_global} and \eqref{eq:ssn_local}, the sample complexity should be
\begin{small}
	\begin{equation*}
	s \: = \: \Theta \big(  \tfrac{d }{\varepsilon^2 } 
	\, \log \tfrac{d}{\delta} \big)  ,
	\end{equation*}
\end{small}%
which depends on $d$ (worse than ours $\ed$) but does not depend on coherence.
We show that if the {\it ridge leverage score sampling} is used, then $s = \Theta \big(  \tfrac{\ed}{\varepsilon^2} \log \tfrac{\ed}{\delta} \big)$ samples suffices, which is better than the above sample complexity.
However, because approximately computing the (ridge) leverage scores is expensive, neither the leverage score sampling of~\cite{xu2016sub} nor the ridge leverage score sampling proposed by us is a practical choice.

\section{Distributed Newton-Type Method} \label{sec:dist}

Communication-efficient distributed optimization is an important research field, and second-order methods have been developed to reduce the communication cost, e.g., DANE~\cite{shamir2014communication}, AIDE~\cite{reddi2016aide}, DiSCO~\cite{zhang2015disco} and GIANT~\cite{wang2018giant}. 
Among them, GIANT has the strongest convergence bound.
In this section, we further improve the convergence analysis of GIANT and show that GIANT does converge when the local sample size, $s = \frac{n}{m}$, is smaller the number of features, $d$.

\subsection{Motivation and algorithm description}

Assume the $n$ samples are partition among $m$ worker machines {\it uniformly at random}.
Each worker machine has its own processors and memory, and the worker machines can communicate by message passing.
The communication are costly compared to the local computation; 
when the number of worker machines is large, the communication is oftentimes the bottleneck of distributed computing.
Thus there is a strong desire to reduce the communication cost of distributed computing.
Our goal is to solve the optimization problem~\eqref{problem} in a communication-efficient way.

The first-order methods are computation-efficient but not communication-efficient.
Let us take the gradient descent for example.
In each iteration, with the iteration $\w_t$ at hand, the $i$-th worker machine uses its local data to compute a {\it local gradient} $\g_{t,i}$;
Then the driver machine averages the local gradient to form the exact gradient $\g_t$ and update the model by
\begin{small}
	\begin{equation*}
	\w_{t+1} \: = \: \w_{t} - \alpha_t \g_t ,
	\end{equation*}
\end{small}%
where $\alpha_t$ is the step size.
Although each iteration is computationally efficient, the first-order methods (even with acceleration) take many iterations to converge, especially when the condition number is big.
As each iteration requires broadcasting $\w_t$ and an aggregation of the local gradients to form $\g_t$, the total number and complexity of communication are big.

Many second-order methods have been developed to improve the communication-efficiency, among which 
the Globally Improved Approximate NewTon (GIANT) method~\cite{wang2018giant} has the strongest convergence rates.
Let $s = \frac{n}{m}$ be the local sample size and 
$\A_{t,i} \in \RB^{s\times d}$ be the block of $\A_t \in \RB^{n\times d}$, which is previously defined in \eqref{eq:A}, formed by the $i$-th worker machine.
With the iteration $\w_t$ at hand, the $i$-th worker machine can use its local data samples to form the {\it local Hessian matrix}
\begin{small}
\begin{equation*}
\TH_{t,i} 
\: = \: \tfrac{1}{s} \A_{t,i}^{T}\A_{t, i} + \gamma \I_d 
\end{equation*}
\end{small}%
and outputs the local {\it Approximate NewTon (ANT)} direction
\begin{small}
	\begin{equation} \label{eq:giant_sub}
	\Tp_{t,i}
	\: = \:  \TH_{t,i}^{-1}\g_t.
	\end{equation}
\end{small}%
Finally, the driver machine averages the ANT direction
\begin{small}
	\begin{equation*}
	\Tp_{t}
	\: = \: \frac{1}{m} \sum_{i=1}^m \Tp_{t,i}
	\end{equation*}
\end{small}%
and perform the update
\begin{small}
	\begin{equation*}
	\w_{t+1} \: = \: \w_{t} - \alpha_t \Tp_t ,
	\end{equation*}
\end{small}%
where the step size $\alpha_t$ can be set to one under certain conditions;
we only consider the $\alpha_t$ case in the rest of this section.

GIANT is much more communication-efficient than the first-order methods.
With $\alpha_t$ fixed, each iteration of GIANT has four rounds of communications: 
(1) broadcasting $\w_t$, (2) aggregating the local gradients to form $\g_t$, (3) broadcasting $\g_t$, and (4) aggregating the ANT directions to form $\Tp_t$; thus the per-iteration communication cost is just twice as much as a first-order method.
\cite{wang2018giant} showed that GIANT requires a much smaller number of iterations than the accelerated gradient method which has the optimal iteration complexity (without using second-order information).

\subsection{Our improved convergence bounds}

We analyze the GIANT method and improve the convergence analysis of~\cite{wang2018giant}, which was the strongest theory in terms of communication efficiency.
Throughout this section, we assume the $n$ samples are partitioned to $m$ worker machine uniformly at random.

\paragraph{Improved bound for quadratic loss.}
We let $\w^\star$ be the unique  optimal solution to the problem~\ref{problem} and $\De_t = \w_t - \w^\star$.
If the loss function of \eqref{problem} is quadratic, e.g., $l_i (\x_i^T \w) = \frac{1}{2} (\x_i^T \w - y_i)^2$, the Hessian matrix $\H_t = \frac{1}{n} \A_t^T \A_t + \gamma \I_d$ does not change with the iteration, so we use $\H$ and $\A$ instead.
Theorem~\ref{thm:giant_quad} guarantees the global convergence of GIANT.

\begin{theorem}[Global Convergence]
	\label{thm:giant_quad}
	Let $\ed, \mu^{\gamma}$ respectively be the $\gamma$-ridge leverage score and $\gamma$-coherence of $\A$, and $\kappa$ be the condition number of $\H$.
	Let $\varepsilon \in (0, \frac{1}{4})$ and $\delta \in (0, 1)$ be any user-specified constants.
	Assume the loss function of \eqref{problem} is quadratic.
	For a sufficiently large sub-sample size:
	\begin{small}
		\begin{equation*}
		s \: = \: \Theta \Big(  \tfrac{\mu^{\gamma}\ed}{\varepsilon} \log \tfrac{m\ed}{\delta} \Big)  ,
		\end{equation*}
	\end{small}%
	with probability at least $ 1-\delta$,
	\begin{small}
		\begin{equation} \label{eq:giant_global}
		\big\| \De_t \big\|_2
		\: \leq \: \varepsilon^t \sqrt{\kappa}  \, \big\| \De_0 \big\|_2 .
		\end{equation}
	\end{small}
\end{theorem}

\begin{proof}
	We prove the theorem in Appendix~\ref{ap:giant_quad}.
\end{proof}

\paragraph{Improved bound for non-quadratic loss.}
If the loss function of \eqref{problem} is non-quadratic, we can only guarantee fast local convergence under Assumption~\ref{am:lipschitz}, as well as the prior works~\cite{wang2018giant}.

\begin{theorem} [Local Convergence]
	\label{thm:giant_nonquad}
	Let $\ed, \mu^{\gamma}$ respectively be the $\gamma$-ridge leverage score and $\gamma$-coherence of $\A_t$.
	Let $\varepsilon \in (0, \frac{1}{4})$ and $\delta \in (0, 1)$ be any user-specified constants.
	Let Assumption~\ref{am:lipschitz} be satisfied.
	For a sufficiently large sub-sample size:
	\begin{small}
		\begin{equation*}
		s \: = \: \Theta \Big(  \tfrac{\mu^{\gamma}\ed}{\varepsilon} \log \tfrac{m\ed}{\delta} \Big)  ,
		\end{equation*}
	\end{small}%
	with probability at least $ 1-\delta$,
	\begin{small}
		\begin{equation} \label{eq:giant_local}
		\big\|\De_{t+1} \big\|_2
		\: \leq \: \varepsilon \, \sqrt{ \kappa_t } \, \big\|\De_{t} \big\|_2
		+  \tfrac{L}{\sigma_{\min}(\H_t)} \, \big\|\De_{t} \big\|_2^2 ,
		\end{equation}
	\end{small}%
	where $\kappa_t = \tfrac{ \sigma_{\max} (\H_t) }{ \sigma_{\min} (\H_t) }$ is the condition number.
\end{theorem}

\begin{proof}
	We prove the theorem in Appendix~\ref{ap:giant_nonquad}
\end{proof}

\begin{remark}
	GIANT is a variant of SSN:
	SSN uses one of $\{ \Tp_{t,i} \}_{i=1}^m$ as the descending direction, whereas GIANT uses the averages of the $m$ directions.
	As a benefit of the averaging, the sample complexity is improved from 
	$s = \tilde{\Theta} \big( \tfrac{\ed }{\epsilon^2} \big)$ to $s = \tilde{\Theta} \big( \tfrac{\ed }{\epsilon} \big)$.
\end{remark}

\subsection{Comparison with prior work}

To guarantee the same convergence bounds, \eqref{eq:giant_global} and \eqref{eq:giant_local},
\citeauthor{wang2018giant} require a sample complexity of $s = \Theta (  \tfrac{\mu^0 d}{\varepsilon} \log \tfrac{d}{\delta} ) $.\footnote{The sample complexity in~\cite{wang2018giant} is actually slightly worse; but it is almost trivial to improve their result to what we showed here.}
This requires require the local sample size $s = \frac{n}{m}$ be greater than $d$, even if the coherence $\mu^0$ is small.
As communication and synchronization costs grow with $m$, the communication-efficient method, GIANT, is most useful for the large $m$ setting;
in this case, the requirement $n > md$ is unlikely satisfied.

In contrast, our improved bounds do not require $n > md$.
As $\ed$ can be tremendously smaller than $d$, our requirement can be satisfied even if $m$ and $d$ are both large.
Our bounds match the empirical observation of~\cite{wang2018giant}: GIANT convergences rapidly even if $m d$ is larger than $n$.

\section{Sub-Sampled Proximal Newton (SSPN)}
\label{sec:sspn}

In the previous sections, we analyze second-order methods for the optimization problem~(\ref{problem}) which has a smooth objective function.
In this section, we study a harder problem:
\begin{small}
	\begin{equation*}
	\min_{\w \in \RB^d} \;
	\frac{1}{n} \sum_{j=1}^n l_j (\x_j^T \w ) + \frac{\gamma }{2} \| \w\|_2^2 + r (\w ) ,
	\end{equation*}
\end{small}%
where $r$ is a non-smooth function.
The standard Newton's method does not apply because the second derivative of the objective function does not exist.
Proximal Newton~\cite{lee2014proximal}, a second-order method, was developed to solve the problem, and later on, sub-sampling was incorporated to speed up computation~\cite{liu2017inexact}.
We further improve the bounds of Sub-Sampled Proximal Newton (SSPN).

\subsection{Algorithm Description}

Let $F (\w ) = \frac{1}{n} \sum_{j=1}^n l_j (\x_j^T \w ) + \frac{\gamma }{2} \| \w\|_2^2 $ be the smooth part of the objective function, and $\g_t$ and $\H_t $ be its first and second derivatives at $\w_t$.
The proximal Newton method~\cite{lee2014proximal} iterative solves the problem:
\begin{small}
	\begin{equation*}
	\p_t \: = \: \argmin_{\p } \tfrac{1}{2} \big( \p^T \H_t \p - 2\g_t^T \p  +  \g_t^T\H_t^{-1} \g_t \big) + r (\w_t - \p) ,
	\end{equation*}
\end{small}%
and then perform the update $\w_{t+1} = \w_t - \p_t$.
The righthand side of the problem is a local quadratic approximation to $F (\w)$ at $\w_t$.
If $r (\cdot ) = 0$, then proximal Newton is the same as the standard Newton's method.

The sub-sampled proximal Newton (SSPN) method uses sub-sampling to approximate $\H_t$;
let the approximate Hessian matrix be $\widetilde{\H}_t$, as previously defined in \eqref{eq:pH}.
SSPN computes the ascending direction by solving the local quadratic approximation:
\begin{small}
	\begin{equation} \label{eq:sspn_sub}
	\Tp_t \: = \: \argmin_{\p } \tfrac{1}{2} \big( \p^T\TH_t \p  - 2\g_t^T \p +  \g_t^T \TH_t^{-1} \g_t \big) + r (\w_t - \p) ,
	\end{equation}
\end{small}%
and then perform the update $\w_{t+1} = \w_t - \Tp_t$.

\subsection{Our improved error convergence bounds}
We show that SSPN has exactly the same iteration complexity as SSN, for either quadratic or non-quadratic function $l_j (\cdot )$.
Nevertheless, the overall time complexity of SSPN is higher than SSN, as the subproblem \eqref{eq:sspn_sub} is expensive to solve if $r (\cdot )$ is non-smooth.

\begin{theorem}\label{thm:sspn}
	Theorems~\ref{thm:ssn_quad}, \ref{thm:ssn_nonquad}, and \ref{thm:ridge} hold for SSPN.
\end{theorem}
\begin{proof}
	We prove the theorem in Appendix~\ref{ap:sspn_quad} and \ref{ap:sspn_nonquad}.
\end{proof}

\subsection{Comparison with prior work}
\cite{liu2017inexact} showed that when $\|\De_{t}\|_2$ is small enough, $\|\De_{t+1}\|_2$ will converge to zero linear-quadratically, similar to our results. But their sample complexity is
\begin{small}
\[ s = \tilde{\Theta} \big( \tfrac{d}{\varepsilon^2}  \big).   \]
\end{small}%
This requires the sample size to be greater than $d$. 
The $\ell_1$ regularization is often used for high-dimensional data, the requirement that $d < s \ll n$ is too restrictive.

Our improved bounds show that $s = \tilde{\Theta} ( \tfrac{\ed \mu^\gamma}{\varepsilon^2})$ suffices for uniform sampling and that $s = \tilde{\Theta} ( \tfrac{\ed}{\varepsilon^2})$ suffices for ridge leverage score sampling.
Since $\ed$ can be tremendously smaller than $d$ when $n\gamma \gg 1$, our bounds are useful for high-dimensional data.

\section{Inexactly Solving the Subproblems} \label{sec:inexact}

Each iteration of SSN (Section~\ref{sec:ssn}) and GIANT (Section~\ref{sec:dist}) involves solving a subproblem in the form
\begin{small}
	\begin{equation*}
	\big( \tfrac{1}{s} \TA_t^T \TA_t + \gamma \I_d \big)\p \: = \: \g_t .
	\end{equation*}
\end{small}%
Exactly solving this problem would perform the multiplication $\TA_t^T \TA_t$ and decompose the $d\times d$ approximate Hessian matrix $\tfrac{1}{s} \TA_t^T \TA_t + \gamma \I_d$; the time complexity is $\OM (s d^2 + d^3)$.
In practice, it can be approximately solved by the conjugate gradient (CG) method, each iteration of which applies a vector to $\TA_t$ and $\TA_t^T$; the time complexity is $\OM (q \cdot \nnz (\A) )$, where $q$ is the number of CG iterations and $\nnz$ is the number of nonzeros.
The inexact solution is particularly appealing if the data are sparse.
In the following, we analyze the effect of the inexact solution of the subproblem.

Let $\kappa_t$ be the condition number of $\TH_t$.
For smooth problems,~\cite{wang2018giant} showed that by performing 
\begin{small}
	\begin{equation*}
	q \: \approx \: \tfrac{ \sqrt{\kappa_t} - 1 }{2} \log \tfrac{8}{\varepsilon_0^2} 
	\end{equation*}
\end{small}%
CG iterations,
the conditions \eqref{eq:cg_condition1} and \eqref{eq:cg_condition2} are satisfied,
and the inexact solution does not much affect the convergence of SSN and GIANT.

\begin{corollary}[SSN] \label{cor:ssn_inexact}
	Let $\Tp_t$ and $\Tp_t'$ be respectively the exact and an inexact solution to the quadratic problem $\TH_t^{-1} \p = \g_t$. 
	SSN  updates $\w$ by $\w_{t+1} = \w_t - \tilde{\p}_t'$.
	If the condition
	\begin{small}
		\begin{equation} \label{eq:cg_condition1}
		\big\| \TH_t^{1/2} \, ( \Tp_t - \Tp_t' ) \big\|_2
		\: \leq \: \tfrac{\varepsilon_0}{2} \big\| \TH_t^{1/2} \,  \Tp_t  \big\|_2
		\end{equation}
	\end{small}%
	is satisfied for some $\varepsilon_0 \in (0, 1)$, 
	then Theorems~\ref{thm:ssn_quad} and \ref{thm:ssn_nonquad}, with $\varepsilon$ in \eqref{eq:ssn_global} and \eqref{eq:ssn_local} replaced by $\varepsilon + \varepsilon_0$, continue to hold.
\end{corollary}

\begin{proof}
	We prove the corollary in Appendix~\ref{ap:inexact:ssn}.
\end{proof}

\begin{corollary}[GIANT] \label{cor:giant_inexact}
	Let $\tilde{\p}_{t,i}$ and $\tilde{\p}_{t,i}'$ be respectively the exact and an inexact solution to the quadratic problem $\TH_{t,i}^{-1} \p = \g_t$. 
	GIANT  updates $\w$ by $\w_{t+1} = \w_t - \frac{1}{m} \sum_{i=1}^m \Tp_{t,i}'$.
	If the condition
	\begin{small}
		\begin{equation} \label{eq:cg_condition2}
		\big\| \TH_{t,i}^{1/2} \, ( \Tp_{t,i}  - \Tp_{t,i} ' ) \big\|_2
		\: \leq \: \tfrac{\varepsilon_0}{2}\big\| \TH_t^{1/2} \,  \Tp_{t,i}  \big\|_2
		\end{equation}
	\end{small}%
	is satisfied for some $\varepsilon_0 \in (0, 1)$ and all $i \in [m]$, 
	then Theorems~\ref{thm:giant_quad} and \ref{thm:giant_nonquad}, with $\varepsilon$ in \eqref{eq:giant_global} and \eqref{eq:giant_local} replaced by $\varepsilon + \varepsilon_0$, continue to hold.
\end{corollary}

\begin{proof}
	The corollary can be proved in almost the same way as~\cite{wang2018giant}.
	So we do not repeat the proof.
\end{proof}

SSPN is designed for problems with non-smooth regularization, in which case finding the exact solution may be infeasible,
and the sub-problem can only be inexactly solved.
If the inexact satisfies the same condition \eqref{eq:cg_condition1},
Corollary~\ref{cor:sspn} will guarantee the convergence rate of SSPN.

\begin{corollary}[SSPN]\label{cor:sspn}
	Let $\tilde{\p}_t$ and $\tilde{\p}_t'$ be respectively the exact and an inexact solution to the non-smooth problem \eqref{eq:sspn_sub}. 
	SSPN updates $\w$ by $\w_{t+1} = \w_t - \tilde{\p}_t'$.
	If $\Tp_t'$ satisfies the condition~\eqref{eq:cg_condition1} for any $\varepsilon_0 \in (0, 1)$, then Theorems~\ref{thm:sspn} still holds for SSPN with $\varepsilon$ replaced by $\varepsilon + \varepsilon_0$.
\end{corollary}

\begin{proof}
	We prove the corollary in Appendix~\ref{ap:inexact:sspn}.
\end{proof}

\section{Conclusion}\label{sec:conclusion}

We studied the subsampled Newton (SSN) method and its variants, GIANT and SSPN, and established stronger convergence guarantees than the prior works.
In particular, we showed that a sample size of $s= \tilde{\Theta} (\ed)$ suffices, where $\gamma$ is the $\ell_2$ regularization parameter and $\ed$ is the effective dimension.
When $n\gamma$ is larger than most of the eigenvalues of the Hessian matrices, $\ed$ is much smaller than the dimension of data, $d$.
Therefore, our work guarantees the convergence of SSN, GIANT, and SSPN on high-dimensional data where $d$ is comparable to or even greater than $n$.
In contrast, all the prior works required a conservative sample size $s = \Omega (d)$ to attain the same convergence rate as ours.
Because subsampling means that $s$ is much smaller than $n$, the prior works did not lend any guarantee to SSN on high-dimensional data.


\begin{small}
	\bibliographystyle{named}
	\bibliography{ED}
\end{small}

\appendix

%
%

\newpage

\appendix

\section{Random Sampling for Matrix Approximation}\label{ap:rand_sample}

Here, we give a short introduction to random sampling and their theoretical properties.
Given a matrix $\A \in \RB^{n \times d}$, row selection constructs a smaller size matrix $\C \in \RB^{s \times d}$ ($s < n$) as an approximation of $\A$. The rows of $\C$ is constructed using a randomly sampled and carefully scaled subset of the rows of $\A$. Let $p_1, \cdots, p_n \in(0, 1)$ be the sampling probability associated with the rows of $\A$. The rows of $\C$ is selected independently according to the sampling probability $\{p_j\}_{j=1}^n$ such that for all $j \in [n]$, we have
\begin{small}
	\begin{equation*}
	\PB (\c_j = \a_k / \sqrt{sp_k}) = p_k,
	\end{equation*}
\end{small}%
where $\c_j$ and $\a_k$ are the $j$-th row of $\C$ and $k$-th row of $\A$. In a matrix multiplication form, $\C$ can be formed as 
\begin{small}
	\begin{equation*}
	\C = \S^T\A,
	\end{equation*}
\end{small}%
where $\S \in \RB^{\s \times d}$ is called the sketching matrix. As a result of row selection, there is only a non-zero entry in each column of $\S$, whose position and value correspond to the sampled row of $\A$.

\paragraph{Uniform sampling.}
Uniform sampling simply sets all the sampling probabilities equal, i.e., $p_1=\cdots=p_n=\frac{1}{n}$. Its corresponding sketching matrix $\S$ is often called uniform sampling matrix. The non-zero entry in each column of $\S$ is the same, i.e., $\sqrt{\frac{n}{s}}$. 
If $s$ is sufficiently large,
\begin{small}
	\begin{equation} \label{eq:ssn}
	\widetilde{\H}_t 
	\: = \: \tfrac{1}{n} \A_t^T \S \S^T \A_t + \gamma \I_d
	\end{equation}
\end{small}%
is a good approximation to $\H_t$.

\begin{lemma}[Uniform Sampling]\label{lem:snn_uniform_subsampling}
	Let $\H_t $ and $\TH_t $ be defined as that in \eqref{eq:truehessian} and \eqref{eq:ssn}. Denote $\ed = \ed(\A_t), \mu^{\gamma}=\mu^{\gamma}(\A_t)$ for simplicity. Given arbitrary error tolerance $\varepsilon \in (0, 1)$ and failure probability $\delta \in (0, 1)$, when
	\begin{small}
		\begin{equation*}
		s = \Theta\left(  \frac{\mu^{\gamma}\ed}{\varepsilon^2} \log \frac{\ed}{\delta} \right) 
		\end{equation*}
	\end{small}%
	the spectral approximation holds with probability at least $1-\delta$:
	\begin{small}
		\begin{equation*}
		(1-\varepsilon)\H_t  \preceq \TH_t \preceq  (1+\varepsilon)\H_t 
		\end{equation*}
	\end{small}%
\end{lemma}
\begin{proof}
	The proof trivially follows from~\cite{cohen2015ridge}.
\end{proof}

\paragraph{Ridge leverage score sampling.}
It takes $p_j$ proportional to the $j$-th ridge leverage score, i.e., 
\begin{small}
	\begin{equation}\label{eq:leverage_prob}
	p_j =  l_j^{\gamma}/\sum_{i=1}^n l_i^{\gamma},  \quad \forall \; j \; \in [n]
	\end{equation}
\end{small}%
where $l_i^{\gamma}$ is the ridge leverage score of the $i$-th row of $\A$. Let $\U$ be its sketching matrix. Then the non-zero entry in $j$-th column of $\U$ is $\sqrt{\frac{1}{s\cdot p_k}}$ if the $j$-th row of $\U^T\A$ is drawn from the $k$-th row of $\A$, where $p_k$ is defined as \eqref{eq:leverage_prob}. If the ridge leverage score sampling is used to approximate the $d \times d$ Hessian matrix, the approximate Hessian matrix turns to
\begin{small}
	\begin{equation}\label{eq:leverage_appro_H}
	\TH_t 
	\: = \: \tfrac{1}{n} \A_t^T \U \U^T \A_t + \gamma \I_d.
	\end{equation}
\end{small}%
The sample complexity in Theorems~\ref{thm:ssn_quad} and \ref{thm:ssn_nonquad} will be improved to $s = \Theta \big(  \tfrac{\ed}{\varepsilon^2} \log \tfrac{\ed}{\delta} \big)$.

\begin{lemma}[Ridge Leverage Rampling]
	\label{lem:ridge_leverage_sampling}
	Let $\H_t $ and $\TH_t $ be defined as that in \eqref{eq:truehessian} and \eqref{eq:leverage_appro_H}. Denote $\ed = \ed(\A_t), \mu^{\gamma}=\mu^{\gamma}(\A_t)$ for simplicity. Given arbitrary error tolerance $\varepsilon \in (0, 1)$ and failure probability $\delta \in (0, 1)$, when
	\begin{equation*}
	s = \Theta\left(  \frac{\ed}{\varepsilon^2} \log \frac{\ed}{\delta} \right) 
	\end{equation*}
	the spectral approximation holds with probability at least $1-\delta$:
	\begin{equation*}
	(1-\varepsilon)\H_t  \preceq \TH_t \preceq  (1+\varepsilon)\H_t 
	\end{equation*}
\end{lemma}

\begin{proof}
	The proof trivially follows from~\cite{cohen2015ridge}.
\end{proof}

\section{Convergence of Sub-Sampled Newton}\label{ap:snn}

In this section, we first give a framework of analyzing the recursion of $\De_{t} = \w_t-\w^*$, which also inspires the proofs for distributed Newton-type Method and SSPN. Within this simple framework, we then complete the proofs for the global and local convergence for SSN. 

\subsection{A analyzing framework}\label{ap:snn_framework}


\paragraph{Approximate Newton Direction.}
We can view the process of solving the newton direction $\p_t$ from the linear system
\begin{small}
	\begin{equation} \label{eq:snn_appro_linear_system}
	\big( \tfrac{1}{s} \A_t \A_t^T + \gamma \I_d \big) \, \p \: = \: \g_t
	\end{equation}
\end{small}%
as a convex optimization. Recalling that $\A_t$ is defined in \eqref{eq:A}, we define the quadratic auxiliary function
\begin{small}
	\begin{equation}\label{eq:snn_auxiliary}
	\phi_t(\p) \triangleq \p^T
	\underbrace{  ( \frac{1}{n}\A_t^{T}\A_t + \gamma \I_d  )}_{\triangleq \H_t}
	\p - 2 \p^{T} \g_t.
	\end{equation}
\end{small}
Obviously, the true Newton direction $\p_t$ is the critical point of $\phi_t(\p)$:
\begin{small}
	\begin{equation}\label{eq:snn_true_direc}
	\p_t^* = \arg\min_{\p} \phi_t(\p) = \H_t^{-1}\g_t
	\end{equation}
\end{small}

Since we use subsampled Hessian $\TH_t$, we solve the approximate Newton direction $\Tp_t$ from~\eqref{eq:ssn_linear_system} instead of~\eqref{eq:snn_appro_linear_system}, thus the counterpart of $\phi_t(\p)$ is defined 
\begin{small}
	\begin{equation}\label{eq:snn_auxiliary_appro}
	\tilde{\phi}_t(\p) \triangleq \p^T
	\underbrace{  ( \frac{1}{n}\A_t^{T}\S_t\S_t^{T}\A_t + \gamma \I_d  )}_{\triangleq \TH_t}
	\p - 2 \p^{T} \g_t.
	\end{equation}
\end{small}
It is easy to verify the approximate Newton  direction $\Tp_t$ is the minimizers of~\eqref{eq:snn_auxiliary_appro}, i.e.,
\begin{small}
	\begin{equation}\label{eq:snn_appro_direc}
	\Tp_t = \argmin_{\p} \tilde{\phi}_t(\p) = \TH_t^{-1}\g_t.
	\end{equation}
\end{small}

Lemma \ref{lem:ssn_approximate_newton_direc} shows that $\Tp_t$ is close to $\p_t$ in terms of the value of $\phi(\cdot)$, if the subsampled Hessian $\TH_t$, which is used to establish the linear system $\Tp_t$ satisfies, is a good spectral approximation of $\H_t$.

\begin{lemma}[Approximate Newton Direction]\label{lem:ssn_approximate_newton_direc}
	Assume $(1-\varepsilon)\H_t  \preceq \TH_t \preceq  (1+\varepsilon)\H_t$ holds already. Let $\phi_t(\p)$, $\p_t^*$ and $\Tp_t$ be defined respectively in~\eqref{eq:snn_auxiliary},~\eqref{eq:snn_true_direc} and~\eqref{eq:snn_appro_direc}. It holds that
	\begin{small}
		\begin{equation*}
		\min_{\p} \phi_t(\p) \le \phi_t(\tilde{\p}_t) \le (1-\alpha^2)\cdot \min_{\p} \phi_t(\p),
		\end{equation*}
	\end{small}
	where $\alpha = \frac{\varepsilon}{1-\varepsilon}$.
\end{lemma}
\begin{proof}
	Since we now analyze the approximate Newton direction locally, we leave out all subscript $t$ for simplicity.
	By the assumption that $(1-\varepsilon)\H  \preccurlyeq \TH \preccurlyeq  (1+\varepsilon)\H$, we conclude that there must exist a symmetric matrix $\Gam $ such that
	\begin{small}
		\begin{equation*}
		\H^{\frac{1}{2}}\TH^{-1}\H^{\frac{1}{2}} \triangleq \I_d + \Gam \quad  \text{and} \quad  -\frac{\varepsilon}{1+\varepsilon}\I_d \preceq \Gam \preceq \frac{\varepsilon}{1-\varepsilon} \I_d. 
		\end{equation*}
	\end{small}
	
	By the definition of $\p^*$ and $\tilde{\p}$, we have
	\begin{small}
		\begin{align*}
		\H^{\frac{1}{2}}(\tilde{\p} - \p^*) 
		&=  \H^{\frac{1}{2}}(\TH^{-1} -\H^{-1})\g \\
		&= \H^{\frac{1}{2}}\H^{-1}(\H-\TH)\TH^{-1}\g\\
		&= \underbrace{  [\H^{-\frac{1}{2}}(\H - \TH)\H^{-\frac{1}{2}} ] }_{\triangleq \Ome} \underbrace{  [\H^{\frac{1}{2}}\TH^{-1}\H^{\frac{1}{2}} ] }_{\triangleq \I_d + \Gam}[\H^{-\frac{1}{2}}\g]\\
		&= \Ome(\I_d + \Gam) \H^{\frac{1}{2}}\p^*,
		\end{align*}
	\end{small}
	where the second equation is the result of $\A^{-1} -\B^{-1} = \B^{-1}(\B-\A)\A^{-1}$ for nonsingular matrixs $\A$ and $\B$. The last equation holds since $\H^{\frac{1}{2}}\p^* =\H^{-\frac{1}{2}}\g$.
	
	It follows that
	\begin{small}
		\begin{align*}
		\big\|\H^{\frac{1}{2}}(\tilde{\p} - \p^*)\big\|_2 
		&\le \big\|\Ome(\I_d + \Gam)\big\| \big\| \H^{\frac{1}{2}}\p^*\big\|_2   \\
		&\le \big\|\Ome\big\| (1 + \big\|\Gam\big\| )\big\|  \H^{\frac{1}{2}}\p^*\big\|_2\\
		&\le \frac{1}{1-\varepsilon}\big\|\Ome\big\|\big\|  \H^{\frac{1}{2}}\p^*\big\|_2 \\
		&\le \frac{\varepsilon}{1-\varepsilon}\big\|  \H^{\frac{1}{2}}\p^*\big\|_2, \numberthis \label{eq:snn_lem_upperbound1}
		\end{align*}
	\end{small}
	where the third inequality follows from $ \big\|\Gam\big\| \le \frac{\varepsilon}{1-\varepsilon} $ and the last inequality holds due to $\big\|\Ome\big\| \le \varepsilon$.
	
	Thus it follows from $\phi(\p^*) = - \big\|\H^{\frac{1}{2}}\p^*\big\|_2^2$ and the definition of $\phi(\Tp)$ that
	\begin{small}
		\begin{align*}
		\phi(\Tp) - \phi(\p^*) 
		&=  \big\|\H^{\frac{1}{2}}\Tp\big\|_2^2 - 2\g^{T}\p + \big\|\H^{\frac{1}{2}}\p^*\big\|_2^2 \\
		&= \big\|\H^{\frac{1}{2}}(\Tp-\p^*)\big\|_2^2. \\  \numberthis \label{eq:snn_lem_upperbound2}
		\end{align*}
	\end{small}
	Combining~\eqref{eq:snn_lem_upperbound2} and~\eqref{eq:snn_lem_upperbound1}, we have that
	\begin{small}
		\[ \phi(\widetilde{\p}) - \phi(\p^*) 
		\le \alpha^2 \big\|\H^{\frac{1}{2}}\p^*\big\|_2 
		= -\alpha^2 \phi(\p^*),  \]
	\end{small}
	where $\alpha = \frac{\varepsilon}{1-\varepsilon}$. Then the lemme follows.
\end{proof}

\paragraph{Approximate Newton Step.}
If $\Tp_t$ is very close to $\p_t^*$ (in terms of the value of the auxiliary function $\phi_t(\cdot)$), then the direction $\Tp_t$, along which the parameter $\w_t$ will descend, can be considered provably as a good along  direction. Provided that $\Tp_t$ is a good descending direction, Lemma \ref{lem:snn_appro_newton_step} establishes the recursion of $\De_t = \w_t - \w^*$ after one step of direction descend.

\begin{lemma}[Approximate Newton Step]
	\label{lem:snn_appro_newton_step}
	Let Assumption~\eqref{am:lipschitz} (i.e., the Hessian matrix is $L$-Lipschitz) hold. Let $\alpha \in (0, 1)$ be any fixed error tolerance. If $\tilde{\p}_t$ satisfies
	\begin{small}
		\[   \phi_t(\tilde{\p}_t) \le (1-\alpha^2)\cdot \min_{\p} \phi_t(\p) \]
	\end{small}
	Then $\De_t = \w_t - \w^*$ satisfies
	\begin{small}
		\[  \De_{t+1}^{T} \H_t \De_{t+1} \le L \big\|\De_t\big\|_2^2 \big\|\De_{t+1}\big\|_2 + \frac{\alpha^2}{1-\alpha^2} \De_t^{T}\H_t \De_t. \]
	\end{small}
\end{lemma}
\begin{proof}
	See the proof of Lemma 9 in~\cite{wang2018giant}.
\end{proof}

\paragraph{Error Recursion.} By combining all the lemmas above, we can analyze the recursion of $\De_{t} = \w_t-\w^*$ for SSN.

Let $\ed$ and $ \mu^{\gamma}$ respectively be the $\gamma$-ridge leverage score and $\gamma$-coherence of $\A_t$. From Lemma \ref{lem:snn_uniform_subsampling}, when $s = \Theta\left(  \frac{\mu^{\gamma}\ed}{\varepsilon^2} \log \frac{\ed}{\delta} \right)$,  $\TH_t$ is a $\varepsilon$ spectral approximation of $\H_t$. By Lemma \ref{lem:ssn_approximate_newton_direc}, the approximate Newton direction $\Tp_t$, solved from the linear system $\TH_t \p = \g_t$, is not far from $\p_t$ in terms of the value of $\phi(\cdot)$ with $\alpha = \frac{\varepsilon}{1-\varepsilon}$. It then follows from Lemma \ref{lem:snn_appro_newton_step} that 
\begin{small}
	\begin{equation}\label{eq:snn_error_recursion}
	\De_{t+1}^{T} \H_t \De_{t+1} \le L \big\|\De_t\big\|_2^2 \big\|\De_{t+1}\big\|_2 + \frac{\alpha^2}{1-\alpha^2} \De_t^{T}\H_t \De_t,
	\end{equation}
\end{small}
which establishes the recursion of $\De_{t}$ for SSN.

\subsection{Proof of SNN for quadratic loss}\label{ap:snn_global}
\begin{proof}[Proof of Theorem \ref{thm:ssn_quad}]
	Since the loss is quadratic, w.l.o.g, let $\H_t \equiv \H$ and $\A_t \equiv \A$. Let $\ed$ and $ \mu^{\gamma}$ respectively be the $\gamma$-ridge leverage score and $\gamma$-coherence of $\A$, and $\kappa$ be the condition number of $\H$. Note that $L \equiv 0$ due to the quadratic loss. Let $\beta = \frac{\alpha}{\sqrt{1-\alpha^2}}$. Since $\varepsilon \le \frac{1}{4}$, then $\beta  \le \sqrt{2}\varepsilon$.
	
	From the last part of the analysis in \ref{ap:snn_framework}, $\De_t = \w_t - \w^*$ satisfies the error recursion inequality \eqref{eq:snn_error_recursion} with $L = 0$, i.e., 
	\begin{small}
		\begin{equation*}
		\De_{t+1}^{T} \H \De_{t+1} \le  \beta^2 \De_t^{T}\H\De_t
		\end{equation*}
	\end{small}
	By recursion, it follows that
	\begin{small}
		\begin{equation*}
		\De_{t+1}^{T} \H \De_{t+1}   \le \beta^{2(t+1)}\De_0^{T}\H \De_0. 
		\end{equation*}
	\end{small}
	Then the theorem follows.
\end{proof}

\subsection{Proof of SNN for non-quadratic loss}\label{ap:snn_local}
\begin{proof}[Proof of Theorem \ref{thm:ssn_nonquad}]
	Let $\ed = \ed(\A_t), \mu^{\gamma}=\mu^{\gamma}(\A_t)$, $\alpha=\frac{\varepsilon}{1-\varepsilon}$ and $\beta = \frac{\alpha}{\sqrt{1-\alpha^2}}$. Since $\varepsilon \le \frac{1}{4}$, then $\beta  \le \sqrt{2}\varepsilon$.
	
	From the last part of the analysis in \ref{ap:snn_framework}, $\De_t = \w_t - \w^*$ satisfies the error recursion inequality \eqref{eq:snn_error_recursion}, i.e., 
	\begin{small}
		\[ 	\De_{t+1}^{T} \H_t \De_{t+1} \le L \big\|\De_t\big\|_2^2 \big\|\De_{t+1}\big\|_2 + \beta^2\De_t^{T}\H_t \De_t. \]
	\end{small}
	
	Let $\kappa_t = \tfrac{ \sigma_{\max} (\H_t) }{ \sigma_{\min} (\H_t) }$ is the condition number. By plugging  $\H_t \preceq \sigma_{\max}(\H_t) \I_d$ and $\sigma_{\min}(\H_{t})\I_d \preceq \H_t$ into \eqref{eq:snn_error_recursion}, it follows that
	\begin{small}
		\begin{equation}
		\label{eq:snn_local_quadratic_inequality}
		\big\|\De_{t+1}\big\|_2^2 -  \frac{L \big\|\De_t\big\|_2^2}{\sigma_{\min}(\H_t)} \big\|\De_{t+1}\big\|_2  -\beta^2 \kappa_t\big\|\De_{t}\big\|_2^2 \le 0.
		\end{equation}
	\end{small}

	Vewing \eqref{eq:snn_local_quadratic_inequality} as a one-variable quadratic inequality about $\big\|\De_{t+1}\big\|_2$ and solving it, we have
	\begin{small}
		\begin{align*}
		&\big\|\De_{t+1}\big\|_2  \\
		&\le \frac{L \big\|\De_t\big\|_2^2}{2 \sigma_{\min}(\H_t)} + \sqrt{  \left[ \frac{L \big\|\De_t\big\|_2^2}{2\sigma_{\min}(\H_t)} \right]^2 +	  \beta^2 \kappa_t\big\|\De_{t}\big\|_2^2}  \\
		&\le \frac{L \big\|\De_t\big\|_2^2}{2 \sigma_{\min}(\H_t)} + \sqrt{  \left[ \frac{L \big\|\De_t\big\|_2^2}{2\sigma_{\min}(\H_t)} \right]^2} +	\sqrt{  \beta^2 \kappa_t\big\|\De_{t}\big\|_2^2}  \\
		&\le \frac{L \big\|\De_t\big\|_2^2}{\sigma_{\min}(\H_t)} +\varepsilon\sqrt{2\kappa_t}\big\|\De_{t}\big\|_2,
		\end{align*}
	\end{small}%
	where the second inequality follows from $\sqrt{a+b} \le \sqrt{a} + \sqrt{b}$ for $a, b \ge 0$ and the last inequality holds by reorganizations and the fact $\beta \le \sqrt{2}\varepsilon$ . Then the theorem follows.
\end{proof}

\subsection{Proof of Theorem~\ref{thm:ridge}} \label{ap:ssn:ridge}

Theorem~\ref{thm:ridge} can be proved in the same way as Theorems~\ref{thm:ssn_quad} and \ref{thm:ssn_nonquad}; the only difference is using Lemma~\ref{lem:ridge_leverage_sampling} instead of Lemma~\ref{lem:snn_uniform_subsampling}

\section{Convergence of GIANT}\label{ap:dntm}

We still use the framework described in Appendix \ref{ap:snn_framework} to prove the results for GIANT. But two modification should be made in the proof. The first one lies in the part of analyzing Uniform Sampling, since data are distributed and only accessible locally and subsampled Hessians are constructed locally. We can prove each worker machine can simultaneously obtain a $\varepsilon$ spectral approximation of the Hessian matrix. The second one lies in the part of analyzing Approximate Newton Direction, since GIANT uses the global Newton direction, which is the average of all local ANT directions, to update parameters. We can prove the global Newton direction is still a good descending direction.  Once above two modifications are solid established, we prove the main theorems for GIANT.

\subsection{Two modifications}\label{ap:dntm_framework}

\paragraph{Simultaneous Uniform Sampling.}
We can assume these $s$ samples in each worker machine are randomly draw from $\{(\x_i, l_i)\}_{i=1}^n$. This assumption is reasonable because if the samples are i.i.d. drawn from some distribution, then a data-independent partition can be viewed as uniformly sampling equivalently.

Recall that $\A_{t, i} \in \RB^{s \times d}$ contains the rows of $\A_{t}$ selected by $i$-th work machine in iteration $t$. Let $\S_{i} \in \RB^{ n \times s}$ be the associated uniform sampling matrix with each column only one non-zero number $\sqrt{\frac{n}{s}}$. Then $\A_{t, i} =  \sqrt{\frac{s}{n}}\S_{i}^{T} \A_{t} $. The $i$-th local subsampled Hessian matrix is formed as
\begin{small}
	\begin{equation}\label{eq:dntm_local_H}
	\TH_{t,i} = \frac{1}{s} \A_{t,i}^{T}\A_{t, i} + \gamma \I_d = \frac{1}{n} \A_{t}^{T}\S_{i} \S_{i}^{T} \A_{t}  + \gamma \I_d,
	\end{equation}
\end{small}

\begin{lemma}[Simultaneous Uniform Sampling]
	\label{lem:dntm_s_uniform_sampling}
	Let $\varepsilon, \delta \in (0, 1)$ be fixed parameters. Let $\H_{t}, \TH_{t, i}$ be defined in \eqref{eq:truehessian} and \eqref{eq:dntm_local_H}. Denote $\ed = \ed(\A_t), \mu^{\gamma}=\mu^{\gamma}(\A_t)$ for simplicity. Then when
	\begin{small}
		\[ 	s = \Theta\left(  \frac{\mu^{\gamma}\ed}{\varepsilon^2} \log \frac{\ed m}{\delta} \right)  \]
	\end{small}
	with probability at least $1-\delta$, the spectral approximation holds simultaneously, i.e., 
	\begin{small}
		\begin{equation}\label{eq:dntm_uniform_sampling_local}
		\forall i \in [m], \quad (1-\varepsilon)\H_{t}  \preceq \TH_{t,i} \preceq  (1+\varepsilon)\H_{t}.
		\end{equation}
	\end{small}
\end{lemma}

\begin{proof}
	Since we analyze each local $\A_t$, we leave out the subscribe $t$ for simplicity.
	
	By Lemma \ref{lem:snn_uniform_subsampling}, we know with probability $1-\frac{\delta}{m}$, when
	$s = \Theta\left(  \frac{\mu^{\gamma}\ed}{\varepsilon^2} \log \frac{\ed m}{\delta} \right) $, it follows that ,
	\begin{small}
		\[ (1-\varepsilon)\H  \preceq \TH_{i} \preceq  (1+\varepsilon)\H, \quad \forall i \in [m]. \]
	\end{small}

	By Bonferroni's method, we know with probability $1-\delta$, the spectral approximation holds simultaneously.
\end{proof}

\paragraph{Global Approximate Newton Direction}
Recall that the gradient at iteration $t$ is $\g_t = \nabla F(\w_t)$. The local ANT computed by $i$-th worker machine is $\Tp_{t, i}  = \TH_{t, i}^{-1}\g_t$. The global Newton direction is formed as 
\begin{small}
	\begin{equation}\label{eq:dntm_global_ant}
	\Tp_t = \frac{1}{m}\sum_{i=1}^m \Tp_{t,i} =\frac{1}{m}\sum_{i=1}^m  \TH_{t, i}^{-1}\g_t = \TH_{t}^{-1}\g_t,
	\end{equation}
\end{small}
where $ \TH_{t}$ is defined as the harmonic mean of $\TH_{t,i}$, i.e.,
\begin{small}
	\begin{equation*}
	\TH_t \triangleq \left( \frac{1}{m} \sum_{i=1}^m \TH_{t,i}^{-1} \right)^{-1}.
	\end{equation*}
\end{small}

\begin{lemma}[Model average]
	\label{lem:dntm_global_ant}
	Assume condition~\eqref{eq:dntm_uniform_sampling_local} holds for given $\varepsilon, \delta \in (0, 1)$. Let $\phi_t(\p)$ and $\Tp_t$ be defined respectively in~\eqref{eq:snn_auxiliary} and~\eqref{eq:dntm_global_ant}. It holds that
	\[   \min_{\p} \phi_t(\p) \le \phi_t(\Tp_t) \le (1-\alpha^2)\cdot \min_{\p} \phi_t(\p)  \]
	where $\alpha =\frac{\varepsilon^2}{1-\varepsilon}$.
\end{lemma}
\begin{proof}
	We leave out all subscript $t$ for simplicity. It follows from condition~\eqref{eq:dntm_uniform_sampling_local} that there must exist a symmetric matrix $\Gam_i $ such that
	\begin{small}
		\[  \H^{\frac{1}{2}}\TH_i^{-1}\H^{\frac{1}{2}} \triangleq \I_d + \Gam_i \quad  \text{and} \quad  -\frac{\varepsilon}{1+\varepsilon}\I_d \preceq \Gam_i \preceq \frac{\varepsilon}{1-\varepsilon} \I_d. \]	
	\end{small}

	By the definition of $\Tp_i$ and $\p^*$, we have 
	\begin{small}
		\begin{align*}
		\H^{\frac{1}{2}}(\tilde{\p_i} - \p^*) 
		&=  \H^{\frac{1}{2}}(\TH_i^{-1} -\H^{-1})\g \\
		&= \H^{\frac{1}{2}}\H^{-1}(\H-\TH_i)\TH_i^{-1}\g\\
		&= \underbrace{  [\H^{-\frac{1}{2}}(\H - \TH_i)\H^{-\frac{1}{2}} ] }_{\triangleq \Ome_i} \underbrace{  [\H^{\frac{1}{2}}\TH_i^{-1}\H^{\frac{1}{2}} ] }_{\triangleq \I_d + \Gam_i}[\H^{-\frac{1}{2}}\g]\\
		& = \Ome_i(\I_d + \Gam_i) \H^{\frac{1}{2}}\p^*,  
		\end{align*}
	\end{small}
	where the second equation is the result of $\A^{-1} -\B^{-1} = \B^{-1}(\B-\A)\A^{-1}$ for nonsingular matrixs $\A$ and $\B$ and the last equation holds since $\H^{\frac{1}{2}}\p^* =\H^{-\frac{1}{2}}\g$.
	
	It follows that
	\begin{small}
		\begin{align*}
	\big\| \H^{\frac{1}{2}}
	& (\Tp-\p^*) \big\|_2 \le \big\|  \frac{1}{m} \sum_{i=1}^m \Ome_i(\I_d + \Gam_i)  \big\|  \big\|\H^{\frac{1}{2}}\p^*\big\|_2 \\
	&\le  \left(   \big\|  \frac{1}{m} \sum_{i=1}^m \Ome_i \big\| + \frac{1}{m}\sum_{i=1}^m  \big\|\Ome_i\big\|\big\|\Gam_i\big\| \right)\big\|\H^{\frac{1}{2}}\p^*\big\|_2  \numberthis \label{eq:dntm_upbound_1}
	\end{align*}
	\end{small}

	It follows from the assumption~\eqref{eq:dntm_uniform_sampling_local} that
	\begin{equation}\label{eq:dntm_upbound_2}
	\big\|\Ome_i\big\| \le \varepsilon \quad \text{and} \quad \big\|\Gam_i\big\| \le \frac{\varepsilon}{1-\varepsilon} 
	\end{equation}
	
	Let $ \S = \frac{1}{\sqrt{m}}[ \S_1,\cdots,\S_m]$ be the concatenation of $\S_1,\cdots,\S_m$. Then $\S \in \RB^{n \times ms}$ is a uniform sampling matrix which samples $n=ms$ rows. Actually, $\S$ is a permutation matrix, with every row and column containing precisely a single 1 with 0s everywhere else. Therefore, 
	\begin{small}
	\begin{equation}\label{eq:dntm_upbound_3}
	\frac{1}{m}\sum_{i=1}^m \Ome_i =\H^{-\frac{1}{2}}(\A^T\A - \A^T\S\S^T\A)\H^{-\frac{1}{2}} = \bf{0}
	\end{equation}
	\end{small}

	It follows from~\eqref{eq:dntm_upbound_1},~\eqref{eq:dntm_upbound_2} and~\eqref{eq:dntm_upbound_3} that
	\begin{small}
		\begin{equation}\label{eq:dntm_norm_df}
	\big\|\H^{\frac{1}{2}}(\Tp - \p^*)\big\|_2 \le \frac{\varepsilon^2}{1-\varepsilon} \big\|\H^{\frac{1}{2}}
	\p^*\big\|_2
	\end{equation}
	\end{small}

	By the definition of $\phi(\p)$ and~\eqref{eq:dntm_norm_df}, it follows that
	\begin{small}
	\begin{equation*}
		\phi(\Tp) - \phi(\p^*) = \big\|\H^{\frac{1}{2}}(\Tp-\p^*)\big\|_2^2 \le \alpha^2  \big\|\H^{\frac{1}{2}}\p^*\big\|_2^2,
	\end{equation*}
	\end{small}
	where  $\alpha = \frac{\varepsilon^2}{1-\varepsilon}$. Then the lemme follows from $\phi(\p^*)=-\big\|\H^{\frac{1}{2}}\p^*\big\|_2^2$.
\end{proof}

\paragraph{Error Recursion.} 
Plugging above two modifications into the analysis framework described in Appendix ~\ref{ap:snn_framework}, we can analyze the recursion of $\De_{t} = \w_t-\w^*$ for GIANT.

Let $\ed$ and $ \mu^{\gamma}$ respectively be the $\gamma$-ridge leverage score and $\gamma$-coherence of $\A_t$. From Lemma \ref{lem:dntm_s_uniform_sampling}, when $s = \Theta\left(  \frac{\mu^{\gamma}\ed}{\varepsilon^2} \log \frac{\ed m}{\delta} \right)$,  for each $i \in [m]$, $\TH_{t,i}$ is a $\varepsilon$ spectral approximation of $\H_t$. By Lemma \ref{lem:dntm_global_ant}, the global ANT $\Tp_t$, an average of all local APTs $\Tp_{t,i}$, is not far from $\p_t$ in terms of the value of $\phi(\cdot)$ with $\alpha = \frac{\varepsilon^2}{1-\varepsilon}$. It then follows from Lemma \ref{lem:snn_appro_newton_step} that~\eqref{eq:snn_error_recursion} still holds but with $\alpha = \frac{\varepsilon^2}{1-\varepsilon}$, i.e.,
\begin{small}
	\begin{equation*}
	\De_{t+1}^{T} \H_t \De_{t+1} \le L \big\|\De_t\big\|_2^2 \big\|\De_{t+1}\big\|_2 + \frac{\alpha^2}{1-\alpha^2} \De_t^{T}\H_t \De_t,
	\end{equation*}
\end{small}
which establishes the recursion of $\De_{t}$ for GIANT.

\subsection{Proof of GIANT for quadratic loss}\label{ap:giant_quad}
\begin{proof}[Proof of Theorem \ref{thm:giant_quad}]
	Since the loss is quadratic, w.l.o.g, let $\H_t \equiv \H$ and $\A_t \equiv \A$. Let $\ed$ and $ \mu^{\gamma}$ respectively be the $\gamma$-ridge leverage score and $\gamma$-coherence of $\A$, and $\kappa$ be the condition number of $\H$. Note that $L \equiv 0$ due to the quadratic loss. Let $\beta = \frac{\alpha}{\sqrt{1-\alpha^2}}$. Since $\varepsilon \le \frac{1}{2}$, then $\beta  \le 3\varepsilon^2$.

From the last part of the analysis in ~\ref{ap:dntm_framework}, $\De_t = \w_t - \w^*$ satisfies the error recursion inequality~\eqref{eq:snn_error_recursion} with $L = 0$, i.e., 
\begin{small}
	\begin{equation*}
	\De_{t+1}^{T} \H \De_{t+1} \le  \beta^2 \De_t^{T}\H\De_t
	\end{equation*}
\end{small}
By recursion, it follows that
\begin{small}
	\begin{equation*}
	\De_{t+1}^{T} \H \De_{t+1}   \le \beta^{2(t+1)}\De_0^{T}\H \De_0. 
	\end{equation*}
\end{small}
Then the theorem follows.
\end{proof}

\subsection{Proof of GIANT for non-quadratic loss}\label{ap:giant_nonquad}
\begin{proof}[Proof of Theorem \ref{thm:giant_nonquad}]
	Let $\ed = \ed(\A_t), \mu^{\gamma}=\mu^{\gamma}(\A_t)$, $\alpha=\frac{\varepsilon}{1-\varepsilon}$ and $\beta = \frac{\alpha}{\sqrt{1-\alpha^2}}$. Since $\varepsilon \le \frac{1}{2}$, then $\beta  \le 3\varepsilon^2$.
	
	From the last part of the analysis in~\ref{ap:dntm_framework}, $\De_t = \w_t - \w^*$ satisfies the error recursion inequality~\eqref{eq:snn_error_recursion}, i.e., 
	\begin{small}
		\[ 	\De_{t+1}^{T} \H_t \De_{t+1} \le L \big\|\De_t\big\|_2^2 \big\|\De_{t+1}\big\|_2 + \beta^2\De_t^{T}\H_t \De_t. \]
	\end{small}
	
	Let $\kappa_t = \tfrac{ \sigma_{\max} (\H_t) }{ \sigma_{\min} (\H_t) }$ is the condition number. By plugging  $\H_t \preceq \sigma_{\max}(\H_t) \I_d$ and $\sigma_{\min}(\H_{t})\I_d \preceq \H_t$ into~\eqref{eq:snn_error_recursion}, we can obtain a one-variable quadratic inequality about $\big\|\De_{t+1}\big\|_2$, which is almost the same form as~\eqref{eq:snn_local_quadratic_inequality} except the value of $\beta$.  Solving it, we have
	\begin{small}
		\begin{equation*}
		\big\|\De_{t+1}\big\|_2 \le \frac{L \big\|\De_t\big\|_2^2}{\sigma_{\min}(\H_t)} +\beta\sqrt{\kappa_t}\big\|\De_{t}\big\|_2,
		\end{equation*}
	\end{small}

 Then the theorem follows from $\beta  \le 3\varepsilon^2$.
\end{proof}

\section{Convergence of SSPN}
Since the proximal mapping of the non-smooth part $r(\cdot)$ is used, rather than direct gradient descend, the analysis of Approximate Newton Step should be modified. We first introduce two properties of proximal mapping, and then provides the error recursion of $\De_{t} =\w_t-\w^*$ for SSPN. Based on that error recursion, we can prove the global convergence for quadratic loss and the local convergence for non-quadratic loss for SSPN.

\subsection{Proximal mapping}

The definition of the proximal operater is merely for theoretical analysis. So we move it to the appendix. The proximal mapping is defined as
\begin{small}
	\begin{equation*}
	\prox_r^{\Q} (\w) 
	\: = \: \argmin_{\z}  \tfrac{1}{2}\|\z - \w\|_{\PP}^2+r(\z)
	\end{equation*}
\end{small}%
which involve the current point $\w$, the convex (perhaps non-smooth) function $r(\cdot)$, and the SPSD precondition matrix $\Q$. 
The update rule of SSPN is:
\begin{small}
	\begin{align} \label{eq:sspn_subproblem}
	& \w_{t+1} 
	\: = \: \prox_r^{\TH_t} \big(\w_t - \TH_{t}^{-1}\g_t \big) \nonumber \\
	&= \: \argmin_{\z}  \tfrac{1}{2} \big\|\z - (\w_t-\TH_t^{-1}\g_t) \big\|_{\TH_t}^2+ r (\z ).
	\end{align}
\end{small}%
SSN can also be written in this form, with $r (\cdot ) = 0$.

The proximal mapping enjoys the nonexpansiveness property and fixed point property~\cite{lee2014proximal}.

\begin{lemma}[Nonexpansiveness]\label{lem:sspn_nonexpansiveness}
	Let $r: \RB^d \rightarrow \RB$ be a convex function and $\H \in \RB^{d \times d}$ be a SPSD matrix. 
	The scaled proximal mapping is nonexpansive, i.e., for all $\w_1$ and and $\w_2$,
	\begin{small}
		\[  \big\|  \prox_r^\H(\w_1) -\prox_r^\H(\w_2)\big\|_{\H} \le \big\| \w_1 - \w_2\big\|_{\H}.   \]
	\end{small}
\end{lemma}

\begin{lemma}[Fixed Point Property of Minimizers]\label{lem:sspn_fixpoint}
	Let $L: \RB^d \rightarrow \RB$ be convex and twice differential and $r: \RB^d \rightarrow \RB$ be convex. Let $\g^*$ be the gradient of $L$ at $\w^*$, i.e., $\g^*=\nabla L(\w_t)$. Then for any SPSD matrix $\H$, $\w^*$ minimizes $F(\w)=L(\w)+r(\w)$ if and only if 
	\begin{small}
		\[ \w^* = \prox_r^{\H}(\w^*-\H^{-1}\g^*).  \]
	\end{small}
\end{lemma}

\subsection{Analysis of Approximate Newton Step}

\begin{lemma}[Error Recursion]\label{lem:sspn_error}
Assume $(1-\varepsilon)\H_t \preceq \TH_t \preceq (1+\varepsilon) \H_t$ and Assumption ~\ref{am:lipschitz} (i.e., the Hessian Lipschitz continuity) hold. Let $\De_{t} = \w_t -\w^*$, we have
\begin{small}
\[ \big\|\De_{t+1}\big\|_{\H_t} \le \frac{1}{1-\varepsilon} \left(  \varepsilon \big\|\De_{t}\big\|_{\H_{t}} + \frac{L}{2\sqrt{\sigma_{\min}(\H_{t})} } \big\|\De_{t}\big\|_2^2 \right).  \]
\end{small}
\end{lemma}

\begin{proof}
Recall that the updating rule is
\begin{small}
\[\w_{t+1} = \prox_r^{\TH_t}\left(\w_t - \TH_{t}^{-1}\g_t\right).\]
\end{small}

Let $\w^*$ be the minimizer and $\g^*=\nabla F(\w^*)$ be the its gradient of the smooth part of the objective (i.e., $F (\w ) = \frac{1}{n} \sum_{j=1}^n l_j (\x_j^T \w ) + \frac{\gamma }{2} \big\| \w\big\|_2^2 $). It follows that
\begin{small}
\begin{align*}
\big\|&\w_{t+1}-\w^*\big\|_{\TH_t}^2 \\
&=\big\| \prox_r^{\TH_t}\left(\w_t - \TH_{t}^{-1}\g_t\right) -
\prox_r^{\TH_t}\left(\w^* - \TH_{t}^{-1}\g^*\right)\big\|_{\TH_t}^2\\
&\le \big\|\left(\w_t - \TH_{t}^{-1}\g_t\right) -\left(\w^* - \TH_{t}^{-1}\g^*\right)\big\|_{\TH_t}^2\\
&=\big\| \TH_t\left(\w_t - \w^*\right)  -\left(\g_t - \g^*\right) \big\|_{\TH_t^{-1}}^2,  \numberthis \label{eq:sspn_upbound_Dt+1}
\end{align*}
\end{small}
where the first equality results from Lemma \ref{lem:sspn_fixpoint} and the first inequality results from Lemma \ref{lem:sspn_nonexpansiveness}.

Let $\De_{t}=\w_t-\w^*$ and $\Z_t=\H_t \De_{t} - (\g_t-\g^*)$ for short. Then
\begin{small}
\begin{align*}
\big\|&\De_{t+1}\big\|_{\TH_t} \\
&\le \big\|\TH_t \De_{t} - (\g_t-\g^*)\big\|_{\TH_t^{-1}}\\
&\le  \big\|\Z_t\big\|_{\TH_t^{-1}}  +  \big\|(\TH_t - \H_t) \De_{t}\big\|_{\TH_t^{-1}}\\
&\le \frac{1}{\sqrt{1-\varepsilon}} \big\|\Z_t\big\|_{\H_t^{-1}}  + \frac{\varepsilon}{\sqrt{1-\varepsilon}} \big\|\H_t^{\frac{1}{2}}  \De_{t}\big\|_2\\
& \le \frac{1}{\sqrt{\sigma_{\min}(\H_t)}\sqrt{1-\varepsilon}} \big\|\Z_t\big\|_2  + \frac{\varepsilon}{\sqrt{1-\varepsilon}} \big\| \De_{t}\big\|_{\H_t}.  \numberthis \label{eq:sspn_upbound_Dt+1_prox-H}
\end{align*}
\end{small}
The first inequality is a rephrasing of~\eqref{eq:sspn_upbound_Dt+1}. The second inequality follows from the triangle inequality and the definition of $\Z_t$. The third inequality holds due to $\TH_t^{-1} \preceq \frac{1}{1-\varepsilon}\H_t^{-1}$ and $\TH_t -\H_t \preceq \varepsilon \H_t$.

Next we upper bound $\big\|\Z_t\big\|_2$. It follows that
\begin{align*}
\big\|&\Z_t\big\|_2 = \big\|\H_t \De_{t} - (\g_t-\g^*)\big\|_2\\
&=\big\|\H_t \De_{t} - \int_{0}^1 \H(\w^*+\tau(\w_t-\w^*)) d\tau \De_{t}\big\|_2\\
&\le \big\|  \De_{t}\big\|_2 \int_{0}^1 \big\| \H(\w_t) - \H(w^*+\tau(\w_t-\w^*))\big\|_2 d\tau\\
&\le  \big\|  \De_{t}\big\|_2 \int_{0}^1 (1-\tau)L \big\|\w_t-\w^*\big\|_2 d\tau \\
&= \frac{L}{2}\big\|\De_{t}\big\|_2^2.  \numberthis \label{eq:sspn_up_Zt}
\end{align*}
The second line holds due to
\begin{small}
	\[\g_t-\g^*  = \int_{0}^1 \H(\w^*+\tau(\w_t-\w^*)) d\tau \De_{t}.\]
\end{small}
The third line follows from Cauchy inequality and the definition $\H_t = \H(\w_t)$. The last inequality holds due to the Hessian Lipschitz continuity (Assumption ~\ref{am:lipschitz}).

Note that $\big\|\De_{t+1}\big\|_{\H_t} \le \frac{1}{\sqrt{1-\varepsilon}}  \big\| \De_{t+1}\big\|_{\TH_t}.$ Thus the lemma follows from this equality,~\eqref{eq:sspn_upbound_Dt+1_prox-H} and~\eqref{eq:sspn_up_Zt}, i.e.,
\begin{small}
\[\big\|  \De_{t+1}\big\|_{\H_t} \le \frac{1}{1-\varepsilon} \left[ \frac{L}{2\sqrt{\sigma_{\min}(\H_t)}} \big\|\De_t\big\|_2^2 + \varepsilon  \big\| \De_{t}\big\|_{\H_t} \right]. \]
\end{small}
\end{proof}

\subsection{Proof of SSPN for quadratic loss}\label{ap:sspn_quad}
\begin{theorem}[Formal statement of Theorem \ref{thm:sspn} for quadratic loss]
	\label{thm:sspn_global}
	Let $\ed$ and $ \mu^{\gamma}$ respectively be the $\gamma$-ridge leverage score and $\gamma$-coherence of $\A$, and $\kappa$ be the condition number of $\H$.
	Let $\varepsilon \in (0, \frac{1}{4})$ and $\delta \in (0, 1)$ be any user-specified constants.
	Assume each loss function $l_i(\cdot)$ is quadratic. For a sufficiently large sub-sample size:
	\begin{small}
		\begin{equation*}
		s \: = \: \Theta \Big(  \tfrac{\mu^{\gamma}\ed}{\varepsilon^2} \log \tfrac{\ed}{\delta} \Big)  ,
		\end{equation*}
	\end{small}%
	with probability at least $ 1-\delta$,
	\begin{small}
		\begin{equation*}
		\big\| \De_t \big\|_2
		\: \leq \: \epsilon^t \sqrt{\kappa}  \, \big\| \De_0 \big\|_2 .
		\end{equation*}
	\end{small}
\end{theorem}

\begin{proof}[Proof of Theorem \ref{thm:sspn_global}]
Under the same condition as Theorem \ref{thm:ssn_quad}, by Lemma \ref{lem:sspn_error}, it follows that
\begin{small}
	\[\big\|  \De_{t+1}\big\|_{\H_t}\le \frac{1}{1-\varepsilon} \left(\frac{L}{2\sqrt{\sigma_{\min}(\H_t)} } \big\|\De_t\big\|_2^2 + \varepsilon \big\| \De_{t}\big\|_{\H_t} \right), \]
\end{small}
	
Since the loss is quadratic then $L \equiv 0$. Let $\H_t \equiv \H$ and $\A_t \equiv \A$. Let $\ed$ and $ \mu^{\gamma}$ respectively be the $\gamma$-ridge leverage score and $\gamma$-coherence of $\A$, and $\kappa$ be the condition number of $\H$. Then we have
\begin{small}
	\[\big\|  \De_{t+1}\big\|_{\H}\le \frac{\varepsilon}{1-\varepsilon}  \big\| \De_{t}\big\|_{\H}, \]
\end{small}

By recursion, it follows that
\begin{small}
\[ \big\|  \De_{t+1}\big\|_2 \le \beta^t \sqrt{\kappa} \big\| \De_{t}\big\|_2,  \]
\end{small}
where $\beta = \frac{\varepsilon}{1-\varepsilon}$ and $\kappa = \frac{\sigma_{\max}(\H)}{\sigma_{\min}(\H)}$ is the conditional number.
\end{proof}

\subsection{Proof of SSPN for non-quadratic loss}\label{ap:sspn_nonquad}
\begin{theorem} [Formal statement of Theorem \ref{thm:sspn} for non-quadratic loss]
	\label{thm:sspn_local}
	Let $\ed, \mu^{\gamma}$ respectively be the $\gamma$-ridge leverage score and $\gamma$-coherence of $\A_t$.
	Let $\varepsilon \in (0, \frac{1}{4})$ and $\delta \in (0, 1)$ be any user-specified constants.
	Let Assumption~\ref{am:lipschitz} be satisfied.
	For a sufficiently large sub-sample size:
	\begin{small}
		\begin{equation*}
		s \: = \: \Theta \Big(  \tfrac{\mu^{\gamma}\ed}{\varepsilon^2} \log \tfrac{\ed}{\delta} \Big)  ,
		\end{equation*}
	\end{small}%
	with probability at least $ 1-\delta$,
	\begin{small}
		\begin{equation*} 
		\big\|\De_{t+1} \big\|_2
		\: \leq \: \varepsilon \, \sqrt{ \kappa_t } \, \big\|\De_{t} \big\|_2
		+  \tfrac{L}{\sigma_{\min}(\H_t)} \, \big\|\De_{t} \big\|_2^2 ,
		\end{equation*}
	\end{small}%
	where $\kappa_t = \tfrac{ \sigma_{\max} (\H_t) }{ \sigma_{\min} (\H_t) }$ is the condition number.
\end{theorem}

\begin{proof}[Proof of Theorem \ref{thm:sspn_local}]
	By plugging $\sigma_{\min}(\H_{t})\I_d \preceq \H_t \preceq \sigma_{\max}(\H_t) \I_d$ into the result of Lemma \ref{lem:sspn_error}, it follows that
	\begin{small}
		\[\big\|  \De_{t+1}\big\|_2\le \frac{1}{1-\varepsilon} \left(\frac{L}{2\sigma_{\min}(\H_t) } \big\|\De_t\big\|_2^2 + \varepsilon \sqrt{\kappa_t} \big\| \De_{t}\big\|_{\H_t} \right), \]
	\end{small}
where $\kappa_t = \tfrac{ \sigma_{\max} (\H_t) }{ \sigma_{\min} (\H_t) }$ is the condition number. 

Since $\varepsilon \le \frac{1}{2}$, $\frac{1}{1-\varepsilon} $ is bounded by 2. Thus we have
\begin{small}
	\[\big\|  \De_{t+1}\big\|_2\le \frac{L}{\sigma_{\min}(\H_t) } \big\|\De_t\big\|_2 + 2\varepsilon \sqrt{\kappa_t} \big\| \De_{t}\big\|_{\H_t}, \]
\end{small}
which proves this theorem.
\end{proof}

\section{Inexact Solution to Sub-Problems}
The computation complexity can be alleviated when the Conjugate Gradient (CG) method is used to compute the inexact Newton step. This methodology has been discussed and practiced before ~\cite{wang2018giant}. In this section, we prove that SSN and GIANT can benefit from inexact Newton step. What's more, we provide an theoretical bound for inexact solution for SSPN.

The framework described in Appendix \ref{ap:snn_framework} can help us to prove results for SSN and GIANT.
Since CG produces an approximate solution for the linear system which the approximate Newton direction satisfies, the analysis of Approximate Newton Direction in Appendix \ref{ap:snn_framework} should be modified. We can prove that when the inexact solution satisfies the particular stopping condition, it is close to the exact Newton direction $\p_t$ in terms of the value of $\phi_t(\cdot)$.

\subsection{Inexactly solving for SSN} \label{ap:inexact:ssn}
In the $t$-th iteration, the exact solution is $\Tp_{t} = \TH_{t}^{-1}\g_t$, where $\TH_{t}$ is the subsampled Hessian defined in~\eqref{eq:truehessian}. Let $\Tp_{t}'$ be the inexact solution produced by CG. It satisfies the stopping condition~\eqref{eq:cg_condition1}, i.e., 
	\begin{small}
	\begin{equation*}
	\big\| \TH_t^{1/2} \, ( \Tp_t' - \Tp_t ) \big\|_2
	\: \leq \:  \tfrac{\varepsilon_0}{2} \big\| \TH_t^{1/2} \,  \Tp_t  \big\|_2.
	\end{equation*}
\end{small}%
Thus SSN takes inexact Newton direction $\Tp_t'$ to update the parameter $\w_t$ instead of $\Tp_{t}$.

\begin{lemma}[Inexact solution for SSN]\label{lem:ssn_inexact}
	For given $\varepsilon \in (0, 1)$, assume $(1-\varepsilon)\H_t  \preceq \TH_t \preceq  (1+\varepsilon)\H_t$ holds. Let $\phi_t(\p)$ be defined in~\eqref{eq:snn_auxiliary}. Let $\Tp_t'$ be the inexact solution satisfying~\eqref{eq:cg_condition1}. Then it holds that
	\[   \min_{\p} \phi_t(\p) \le \phi_t(\Tp_t) \le (1-\alpha_0^2)\cdot \min_{\p} \phi_t(\p)  \]
	where $ \alpha_0 = \frac{\varepsilon_0+ \varepsilon}{1-\varepsilon-\varepsilon_0}$.
\end{lemma}
\begin{proof}
	We leave out the subscript $t$ for simplicity. 
	
	It follows from the stopping condition~\eqref{eq:cg_condition1} that
	\begin{small}
		\begin{align*}
		\big\|&\H^{\frac{1}{2}}\left(\Tp' - \Tp\right)\big\|_2^2 = \left(\Tp' - \Tp\right)^{T}\H\left(\Tp' - \Tp\right)\\
		&\le\frac{1}{1-\varepsilon}\left(\Tp' - \Tp\right)^{T}\TH\left(\Tp' - \Tp\right) 
		\le \frac{\varepsilon_0^2}{1-\varepsilon} \Tp^{T}\TH \Tp.
		\end{align*}	
	\end{small}
	
	Since $\H^{\frac{1}{2}}\p^* =\H^{-\frac{1}{2}}\g$, it follows that
	\begin{small}
		\begin{align*}
		\Tp^{T}\TH \Tp & = (\TH^{-1}\g)^{T}\TH (\TH^{-1}\g)\\
		&= \g^{T}\TH^{-1}\g \le\frac{1}{1-\varepsilon} \g^{T}\H^{-1}\g \\
		&= \frac{1}{1-\varepsilon} \big\|\H^{\frac{1}{2}}\p^*\big\|_2^2.
		\end{align*}
	\end{small}
	
	Then it follows that
	\begin{small}
		\begin{align*}
		\big\| &\H^{\frac{1}{2}}\left(\Tp' - \p^* \right)\big\|_2 \le \big\| \H^{\frac{1}{2}}\left(\Tp' - \Tp \right)\big\|_2 + \big\| \H^{\frac{1}{2}}\left(\Tp - \p^* \right)\big\|_2\\
		& \le \frac{\varepsilon_0}{1-\varepsilon} \big\|\H^{\frac{1}{2}}\p^*\big\|_2 + \big\| \H^{\frac{1}{2}}\left(\Tp - \p^* \right)\big\|_2 \\
		& \le \left(  \frac{\varepsilon_0}{1-\varepsilon} + \frac{\varepsilon}{1-\varepsilon}   \right)\big\|\H^{\frac{1}{2}}\p^*\big\|_2,  \numberthis \label{eq:snn_ineaxt_norm_df}
		\end{align*}
	\end{small}
	where the last inequality is due to~\eqref{eq:snn_lem_upperbound1} (which results from Lemma \ref{lem:ssn_approximate_newton_direc}).
	
	By the definition of $\phi(\p)$ and~\eqref{eq:snn_ineaxt_norm_df}, it follows that
	\begin{small}
		\begin{align*}
		\phi(\Tp') - \phi(\p^*) &= \big\|\H^{\frac{1}{2}}(\Tp'-\p^*)\big\|_2^2 \\
		&\le \left(\frac{\varepsilon_0+\varepsilon}{1-\varepsilon} \right)^2 \big\|\H^{\frac{1}{2}}\p^*\big\|_2^2 \\
		& \le \alpha_0^2 \big\|\H^{\frac{1}{2}}\p^*\big\|_2^2 ,
		\end{align*}
	\end{small}
	where $ \alpha_0 = \frac{\varepsilon_0+ \varepsilon}{1-\varepsilon-\varepsilon_0}$.
\end{proof}

\subsection{Inexactly solving for SSPN}\label{ap:inexact:sspn}
When the inexact solution is used, the update rule of SSPN becomes:
\begin{small}
	\begin{align} \label{eq:inexact_sspn_subproblem}
	& \w_{t+1} 
	\: = \: \prox_r^{\TH_t} \big(\w_t - \Tp_t' \big) \nonumber \\
	&= \: \argmin_{\z}  \tfrac{1}{2} \big\|\z - (\w_t-\Tp_t') \big\|_{\TH_t}^2+ r (\z ),
	\end{align}
\end{small}%
where $\Tp_t'$ is the inexact solution. CG produces $\Tp_t'$ via inexactly solving the linear system $\TH_t\p=\g_t$ with stopping condition~\eqref{eq:cg_condition1}, which is equivalent to
\begin{small}
	\begin{equation*}
	\big\|   \Tp_t' - \Tp_t \big\|_{\TH_t}
	\: \leq \:  \tfrac{\varepsilon_0}{2}  \big\| \Tp_t \big\|_{\TH_t}.
	\end{equation*}
\end{small}%

\begin{lemma}[Inexact solution for GIANT]\label{lem:sspn_inexact}
	Assume $(1-\varepsilon)\H_t \preceq \TH_t \preceq (1+\varepsilon) \H_t$ and Assumption ~\ref{am:lipschitz} (i.e., the Hessian Lipschitz continuity) hold. Assume $\varepsilon \le \frac{1}{2}$. Let $\De_{t} = \w_t -\w^*$. Let $\Tp_t'$ be an approximation to the SSPN direction $\Tp_t$ and $\w_{t+1}' = \w_{t} - \Tp_t'$, we have
	\begin{small}
		\[ \big\|\w_{t+1}'-\w^*\big\|_{\H_t} \le \frac{1}{1-\varepsilon_1} \left(  \varepsilon_1 \big\|\De_{t}\big\|_{\H_{t}} + \frac{L}{2\sqrt{\sigma_{\min}(\H_{t})} } \big\|\De_{t}\big\|_2^2 \right).  \]
	\end{small}
where $\varepsilon_1 =2 \varepsilon + \varepsilon_0$.
\end{lemma}

\begin{proof}
By the updating rules $\w_{t+1} = \w_{t} - \Tp_t$ and $\w_{t+1}' = \w_{t} - \Tp_t'$, we obtain
\begin{small}
\begin{align*}
& \big\| \w_{t+1}' - \w^\star \big\|_{\TH_t} \\
& \leq \:  \big\| \w_{t+1} - \w^\star \big\|_{\TH_t} + \big\| \w_{t+1}' - \w_{t+1} \big\|_{\TH_t}  \\
& = \: \big\| \w_{t+1} - \w^\star \big\|_{\TH_t} + \big\| \Tp_{t}' - \tilde{\p}_{t} \big\|_{\TH_t} .
\end{align*}	
\end{small}

It follows from~\eqref{eq:cg_condition1} that 
\begin{small}
	\begin{equation*}
	\big\| \Tp_{t}' - \Tp_{t} \big\|_{\TH_t} 
	\: \leq \: \tfrac{\epsilon_0}{2}  \big\| \tilde{\p}_{t} \big\|_{\tilde{\H}_t} ,
	\end{equation*}
\end{small}
and thus
\begin{small}
	\begin{align*}
	& \big\| \w_{t+1}' - \w^\star \big\|_{\TH_t} \\
	& \leq \: \big\| \w_{t+1} - \w^\star \big\|_{\TH_t} + \tfrac{\varepsilon_0}{2}  \big\| \tilde{\p}_{t} \big\|_{\TH_t}  \\
	& = \: \big\| \w_{t+1} - \w^\star \big\|_{\TH_t} + \tfrac{\varepsilon_0}{2}  \big\| \w_t - \w_{t+1}  \big\|_{\TH_t}  \\
	& = \: \big\| \w_{t+1} - \w^\star \big\|_{\TH_t} + \tfrac{\varepsilon_0}{2}  \big\|( \w_t - \w^\star ) - (\w_{t+1} - \w^\star ) \big\|_{\TH_t}  \\
	& \leq \: (1 + \tfrac{\varepsilon_0}{2}) \big\| \w_{t+1} - \w^\star \big\|_{\TH_t} + \tfrac{\varepsilon_0}{2}  \big\| \w_t - \w^\star \big\|_{\TH_t} \\
	& \leq \: (1 + \tfrac{\varepsilon_0}{2}) \big\| \Delta_{t+1} \big\|_{\TH_t} + \tfrac{\varepsilon_0}{2}  \big\| \Delta_t \big\|_{\TH_t}  .
	\end{align*}
\end{small}
where $\Delta_{t+1} \triangleq \w_t - \w^\star $.

Since $ (1-\varepsilon)\H_t \preceq \TH_t \preceq (1+\varepsilon) \H_t$, it follows that
\begin{small}
	\[ \big\| \w_{t+1}' - \w^\star \big\|_{\H_t}  \le \sqrt{\frac{1+\varepsilon}{1-\varepsilon}} \left[  (1 + \tfrac{\varepsilon_0}{2}) \big\| \Delta_{t+1} \big\|_{\H_t} + \tfrac{\varepsilon_0}{2}  \big\| \Delta_t \big\|_{\H_t}   \right].  \]
\end{small}

It follows from the bound on $\Delta_{t+1}$ (Lemma \ref{lem:sspn_error}) that 
\begin{small}
	\[ \|\De_{t+1}\|_{\H_t} \le \frac{1}{1-\varepsilon} \left(  \varepsilon \|\De_{t}\|_{\H_{t}} + \frac{L}{2\sqrt{\sigma_{\min}(\H_{t})} } \|\De_{t}\|_2^2 \right).  \]
\end{small}

Thus we can obtain
\begin{small}
	\[ \big\| \w_{t+1}' - \w^\star \|_{\H_t} \le   A_{\varepsilon, \varepsilon_0}  \big\| \Delta_{t} \big\|_{\H_t} +B_{\varepsilon, \varepsilon_0}  \frac{    L}{2\sqrt{\sigma_{\min}(\H_{t})} }  \big\| \Delta_{t} \big\|_{\H_t}^2 ,  \]
\end{small}
where $A_{\varepsilon, \varepsilon_0}$ and $B_{\varepsilon, \varepsilon_0}$ are some function of $\varepsilon$ and $ \varepsilon_0$ satisfying
\begin{small}
	\begin{align*}
	A_{\varepsilon, \varepsilon_0} 
	&= \sqrt{\frac{1+\varepsilon}{1-\varepsilon}} \left[  (1 + \tfrac{\varepsilon_0}{2}) \frac{\varepsilon}{1-\varepsilon}  + \tfrac{\varepsilon_0}{2}  \right] \\
	B_{\varepsilon, \varepsilon_0} 
	&= \sqrt{\frac{1+\varepsilon}{1-\varepsilon}}  \frac{1 + \tfrac{\varepsilon_0}{2}}{1-\varepsilon} 
	\end{align*}
\end{small}

Since $\varepsilon \le \frac{1}{2}$, it follows that
\begin{small}
	\[  A_{\varepsilon, \varepsilon_0}  \le  \frac{2\varepsilon+\varepsilon_0}{1-\left(2\varepsilon+\varepsilon_0\right)}  \quad  \text{and}  \quad B_{\varepsilon, \varepsilon_0} \le \frac{1}{1-\left(2\varepsilon + \varepsilon_0\right)}. \]
\end{small}
Then the lemma follows.
\end{proof}

\begin{proof}[Proof of Corollary \ref{cor:sspn}]
Similar to Lemma \ref{lem:sspn_error}, Lemma \ref{lem:ssn_inexact} show that when the inexact SSPN direction $\Tp_t'$ is used, $\w_{t+1}'=\w_t-\Tp_t'$ still quadratic-linearly converges. The only difference in their conclusions is that $\varepsilon$ in Lemma \ref{lem:sspn_error} now is changed into $\varepsilon_1 = 2\varepsilon + \varepsilon_0$ in 
Lemma \ref{lem:ssn_inexact}. In the proof of Theorem \ref{thm:sspn_global} and \ref{thm:sspn_local}, we can replace Lemma \ref{lem:sspn_error} with Lemma \ref{lem:ssn_inexact}, then results still hold for inexactly solving for SSPN, except that the value of $\varepsilon$ is changed into $\varepsilon_1 = 2\varepsilon + \varepsilon_0$. 

Since we can always determine what $\varepsilon$ to choose in advance, we can use a $\frac{\varepsilon}{2}$ spectral approximation of $\H_t$ (which will slightly change the value of $s$ but will not change its order), thus $\varepsilon_1$ will become $\varepsilon + \varepsilon_0$ as Corollary \ref{cor:sspn} states.
Therefore we prove Corollary \ref{cor:sspn}.
\end{proof}

\end{document}